\documentclass[10pt]{article}

\usepackage[utf8]{inputenc}
\usepackage[T1]{fontenc}

\usepackage{epsf}
\usepackage{amsmath}

\allowdisplaybreaks

\usepackage[showframe=false]{geometry}
\usepackage{changepage}

\usepackage{epsfig}
\usepackage{amssymb}

\usepackage{amsthm}
\usepackage{setspace}
\usepackage{cite}
\usepackage{mcite}

\usepackage{algorithmic}  
\usepackage{algorithm}

\usepackage{shadow}
\usepackage{fancybox}
\usepackage{fancyhdr}

\usepackage{color}
\usepackage[usenames,dvipsnames,svgnames,table]{xcolor}
\newcommand{\bl}[1]{\textcolor{blue}{#1}}

\definecolor{mypurple}{rgb}{.4,.0,.5}

\usepackage[hyphens]{url}

\usepackage[colorlinks=true,
            linkcolor=black,
            urlcolor=black,
            citecolor=purple]{hyperref}

\usepackage{breakurl}

\def\y{{\bf y}}

\def\x{{\bf x}}

\def\x{{\mathbf x}}

\def\u{{\bf u}}

\def\x{{\bf x}}
\def\y{{\bf y}}
\def\z{{\bf z}}
\def\q{{\bf q}}

\def\c{{\bf c}}

\def\h{{\bf h}}

\def\bg{\bar{\g}}

\def\cH{{\mathcal H}}
\def\cG{{\mathcal G}}
\def\cZ{{\mathcal Z}}

\def\be{\begin{equation}}
\def\ee{\end{equation}}
\def\ba{\left[\begin{array}}
\def\ea{\end{array}\right]}

\def\u{{\bf u}}

\def\x{{\bf x}}
\def\y{{\bf y}}
\def\z{{\bf z}}
\def\q{{\bf q}}

\def\c{{\bf c}}

\def\p{{\bf p}}

\def\1{{\bf 1}}

\def\g{{\bf g}}
\def\0{{\bf 0}}

\def\erfinv{\mbox{erfinv}}
\def\erf{\mbox{erf}}
\def\erfc{\mbox{erfc}}

\def\mR{{\mathbb R}}
\def\mN{{\mathbb N}}
\def\mE{{\mathbb E}}
\def\mS{{\mathbb S}}
\def\mP{{\mathbb P}}

\def\lp{\left (}
\def\rp{\right )}

\def\y{{\bf y}}

\def\x{{\bf x}}

\def\x{{\mathbf x}}

\def\u{{\bf u}}

\def\x{{\bf x}}
\def\y{{\bf y}}
\def\z{{\bf z}}
\def\q{{\bf q}}

\def\c{{\bf c}}

\def\h{{\bf h}}

\def\cH{{\cal H}}

\def\be{\begin{equation}}
\def\ee{\end{equation}}
\def\ba{\left[\begin{array}}
\def\ea{\end{array}\right]}

\def\u{{\bf u}}

\def\x{{\bf x}}
\def\y{{\bf y}}
\def\z{{\bf z}}
\def\q{{\bf q}}

\def\c{{\bf c}}

\def\p{{\bf p}}

\def\({\left (}
\def\){\right )}

\def\1{{\bf 1}}

\def\q{{\bf q}}

\def\g{{\bf g}}
\def\0{{\bf 0}}

\def\cX{{\mathcal X}}
\def\cY{{\mathcal Y}}
\def\cZ{{\mathcal Z}}
\def\cL{{\mathcal L}}

\def\erfinv{\mbox{erfinv}}

\usepackage{xcolor}
\usepackage{color}

\definecolor{darkgreen}{rgb}{0, 0.4,0}

\definecolor{purplebrown}{rgb}{0.5,0.1,0.6}

\definecolor{ultclupcol}{rgb}{0.1,0.5,0.5}

\definecolor{mytrycolor}{rgb}{0.5,0.7,0.2}

\definecolor{ultclupcola}{rgb}{.5,0,.5}

\definecolor{shadebrown}{rgb}{0.1,0.1,0.9}
\definecolor{lightblue}{rgb}{0.2,0,1}

\usepackage{fancybox}
\usepackage{graphicx}
\usepackage{epstopdf}
\usepackage{epsfig}
\usepackage{wrapfig}
\usepackage{subfigure}

\usepackage{xcolor}
\usepackage{tcolorbox}
\tcbuselibrary{skins}

\newtcbox{\xmybox}{on line,
arc=7pt,
before upper={\rule[-3pt]{0pt}{10pt}},boxrule=0pt,
boxsep=0pt,left=6pt,right=6pt,top=0pt,bottom=0pt,enhanced, coltext=blue, colback=white!10!yellow}

\newtcbox{\xmyboxa}{on line,
arc=7pt,
before upper={\rule[-3pt]{0pt}{10pt}},boxrule=0pt,
boxsep=0pt,left=6pt,right=6pt,top=0pt,bottom=0pt,enhanced, colback=white!10!yellow}

\newtcbox{\xmyboxb}{on line,
arc=7pt,
before upper={\rule[-3pt]{0pt}{10pt}},boxrule=1pt,colframe=darkgreen!100!blue,
boxsep=0pt,left=6pt,right=6pt,top=0pt,bottom=0pt,enhanced, colback=white!10!yellow}

\newtcbox{\xmyboxc}{on line,
arc=7pt,
before upper={\rule[-3pt]{0pt}{10pt}},boxrule=.7pt,colframe=blue!100!blue,
boxsep=0pt,left=6pt,right=6pt,top=0pt,bottom=0pt,enhanced, coltext=blue, colback=white!10!yellow}

\newtcbox{\xmytboxa}{on line,
arc=7pt,
before upper={\rule[-3pt]{0pt}{10pt}},boxrule=.0pt,colframe=pink!50!yellow,
boxsep=0pt,left=6pt,right=6pt,top=0pt,bottom=0pt,enhanced, coltext=white, colback=blue!40!red}

\newtcbox{\xmytboxb}{on line,
arc=7pt,
before upper={\rule[-3pt]{0pt}{10pt}},boxrule=.0pt,colframe=pink!50!yellow,
boxsep=0pt,left=6pt,right=6pt,top=0pt,bottom=0pt,enhanced, coltext=white, colback=white!40!green}

\makeatletter
\newcommand\subsubsubsection{\@startsection{paragraph}{4}{\z@}{-2.5ex\@plus -1ex \@minus -.25ex}{1.25ex \@plus .25ex}{\normalfont\normalsize\bfseries}}
\newcommand\subsubsubsubsection{\@startsection{subparagraph}{5}{\z@}{-2.5ex\@plus -1ex \@minus -.25ex}{1.25ex \@plus .25ex}{\normalfont\normalsize\bfseries}}
\makeatother

\newtheorem{theorem}{Theorem}

\newtheorem{corollary}{Corollary}

\begin{document}

\begin{singlespace}

\title {Capacity of the Hopfield-Hebbian dynamics with correlated patterns  
}
\author{
\textsc{Mihailo Stojnic
\footnote{e-mail: {\tt flatoyer@gmail.com}} }}
\date{}
\maketitle

\centerline{{\bf Abstract}} \vspace*{0.1in}
 
Allowing for small fraction of errors in retrieved stored binary patterns, Hopfield-Hebbian neural network (NN) associative memory dynamics  \cite{Hop82} achieves linearly scaled capacity. The famous proportionality constant  $\alpha_c\approx 0.14$ is obtained in a large dimensional scenario as the ratio of the number of reliably stored patterns $m$ and the patterns' length $n$ under the assumption that uniformly randomly chosen binary patterns are \emph{independent} among themselves. 

While such an idealized scenario is now well understood, more practical scenarios with correlated/noisy  datasets remaining challenging even for classical statistical mechanics replica approaches. We consider precisely such practical settings where each \emph{archetype} pattern is replaced with $M$ correlated \emph{examples} with parameter $r$ governing the quality of correlation. We develop a powerful Random Duality Theory (RDT) based framework to analyze correlated data and  already on the first level of RDT lifting obtain capacity characterizations that remarkably well match behavior exhibited through simulated dynamics. Moreover, characterizations explicitly feature correlation parameters $M$ and $r$, which allows to precisely determine their effects on phase-transitioning behavior of the underlying dynamics. 

Our results are also a well-fitting complement to the existing literature. Namely, to circumvent challenging statistical mechanics analyses, an alternative combination of spectral and replica methods is considered in \cite{AglMFV26}. We confirm that \cite{AglMFV26}'s elegant spectral characterizations are excellent approximations in the low correlated regime. At the same time, we also close the gap in the regimes with higher correlation levels where spectral methods deviate from the practical dynamics performance.

\vspace*{0.25in} \noindent {\bf Index Terms: Hopfield-Hebbian dynamics; Associative memory capacity; Random duality}.

\end{singlespace}

\section{Introduction}
\label{sec:back}

The early spin-glass neural networks (NN) forms  \cite{PasFig77} were superseded by the  Hopfield's auto-associative memories propositions \cite{Hop82}. The main idea centers around the following concept: starting in a vicinity of an already stored/memorized pattern, an iterative dynamics remains trapped by local optima not far away and allows retrieval of the stored pattern. In other words, patterns possess the so-called basins of attraction that simple iterative dynamics can not escape ensuring that a pattern is properly retrieved (and not confused by an unwanted convergence to a different one). The memory capacity -- defined as the maximum allowable number of the stored patterns for which the proper retrieval is possible -- is the key feature that describes the effectiveness of these types of associative memories. As such, it has been thoroughly studied since the appearance of the  Hopfield's introductory paper \cite{Hop82}. We start by surveying some of the key achieved milestones. That will eventually allow us to properly  contextualize our own results.

\subsection{Prior work and  contextualization}
\label{sec:prior}

Trying to model cognitive type of thinking, Hebb considered an associative type of learning principle \cite{Hebb49}. In \cite{Hop82}, Hopfield further connected such learning principles to a basic iterative dynamics. Overwhelming numerical evidence suggested that the proposed dynamics might exhibit a remarkable convergence property, where (rather quickly) one ends up either in one of the $n$-dimensional stored binary patterns or in its a close proximity. In a way, provided that their number is not too high, the stored patterns are observed as stable and surrounded by sizeable basins of attraction, allowing for their efficient retrieval.

Utilizing a simple Gaussian noise argument and considering \emph{independent} patterns, \cite{Hop82} showed that (with probability approaching 1 as $n\rightarrow \infty$) the networks associative capacity (maximal number of reliably stored patterns) is approximately $m\approx \frac{n}{2\log(n)}$. Hopfield observed an even more remarkable property: if a small fraction of errors is allowed in retrieved patterns, then the so-called \emph{linear} capacity scaling $m\sim n$ is possible as well. Moreover, Gaussian noise arguments allowed for a very precise predition of the underlying linear proportionality $\alpha\triangleq \lim_{n\rightarrow \infty}\frac{m}{n}\approx 0.14$ achieved with only $0.4\%$ of errors. Such a set of developments is truly remarkable and to this day remains one of the most fascinating analytical NN results. 

As Hopfield's predictions were basically indistinguishable from the results obtained simulating proposed dynamics, an avalanche of research regarding various related topics followed through the better part of the last four decades. Some of the most relevant analytical results that particularly relate to our own results are discussed below. As almost all of them consider binary patterns statistical context similar or identical to the one from \cite{Hop82}, we skip reemphasizing it. Along the same lines, here and later on throughout the paper, we often skip repetitively stating that the known results hold with high probability (i.e., with probability approaching 1 as the patterns' size/length $n\rightarrow\infty$).

\textbf{\underline{\emph{Exact retrieval:}}} Early Hopfield's noise considerations, were followed by the memory capacity analysis via the statistical physics replica methods in  \cite{AmiGutSom85}. Under the replica symmetry (RS) ansatz, the very same $m\approx \frac{n}{2\log(n)}$ estimate for \emph{exact} retrieval were obtained. The very first set of mathematically rigorous results in this direction arrived with the appearance of \cite{MPRV87}, where sublinear exact retrieval scaling of \cite{Hop82,AmiGutSom85} was matched. In \cite{FST00}, the RS estimate  $\frac{n}{2\log(n)}$ from \cite{AmiGutSom85} was established as a rigorous lower bound.  Results of \cite{MPRV87} were then in a way extended in \cite{KomPat88}  where existence of a nonzero basin of attraction around patterns is shown. Strong numerical evidence suggesting likely impossibility of \emph{exact} retrieval in regimes with the number of patterns $m$ scaling stronger than $\sim \frac{n}{\log(n)}$ was presented in  \cite{MPRV87} as well. This effectively suggested that to achieve a linear scaling, consideration of approximate retrieval (where one allows for a small fraction of errors in retrieved patterns) might be needed. This again was in an excellent agreement with early Hopfield observations \cite{Hop82}.

\textbf{ \underline{\emph{Approximate retrieval:}}} As stated above, Hopfield's rough analytical arguments coupled with strong numerical algorithmic evidence suggested that the precise value of the underlying linear proportionality in approximate retrieval is likely around $\alpha\triangleq \lim_{n\rightarrow \infty}\frac{m}{n}\approx 0.14$. Existence of a basin of attraction around each of the stored patterns emerged as a key concept driving the success of Hopfield propositions.

The so-called AGS (Amit-Gutfreund-Sompolinsky) basin of attraction was analyzed in \cite{AmiGutSom85} via replica methods (assuming RS ansatz).  It was obtained that Hopfield dynamics converges to a local energy (objective) minima within the basin of attraction (thereby staying trapped in a close vicinity of the stored pattern) provided that  $\alpha\approx 0.137906$. \cite{AmiGutSom85} further predicted that potential additional replica symmetry breaking (RSB) corrections might not necessarily be quantitatively substantial. This prediction was confirmed in \cite{SteKuh94}, where negligibly better estimates $\alpha\approx 0.138186$ and $\alpha\approx 0.138187$ were obtained with 1RSB and 2RSB ansatz, respectively. The AGS basins were further analyzed in \cite{FST00} where $\alpha\geq 0.113$ estimate was obtained and in \cite{Stojnichebhop24} where both RS and 1RSB estimates were matched via fully lifted RDT.

As an alternative to the above mentioned AGS basin, the NLT basin was considered in  \cite{Newman88,Louk94,Louk94a,Louk97}.  In addition to the AGS type of energy well around patterns, MLT imposes that the local energy maximum within the well must not be smaller than the one already achieved by the pattern itself. Due to further restrictions, NLT basin capacity estimates are automatically lower bounds for the corresponding AGS ones. Initial  NLT basin capacity estimate $\alpha\geq 0.056$ obtained in \cite{Newman88} was first improved to $\alpha\geq 0.071$ in \cite{Louk94,Louk94a}   and then to $\alpha\geq 0.08$ in  \cite{Tal98}. In  \cite{Stojnichebhop24}, the corresponding first and second lifting level RDT estimates, $\alpha\approx 0.1295$ and $\alpha\approx 0.1298$,  were obtained as well.

\textbf{\underline{\emph{Different learning rules:}}} If one deviates from the Hebbian learning principle, capacity can grow beyond linear $n$ scaling. The Hebbian rule assumes the quadratic (polynomial of degree 2) objective. Alternatively, utilizing an objective based on a general overlap power (polynomial of degree) $p$, the  capacity of \emph{exact}  and \emph{approximate} retrievals scale as $m\sim \frac{n^{p-1}}{\log(n)}$  and   $m\sim n^{p-1}$, respectively   \cite{GarMult87,AAAABGLP23,KroHop16}. On the other hand, relying on the exponential interactions, allows for the exponential capacity growth \cite{DHLUV17,Rametal21,Vasetal17,LucMez24}. One the one hand, these results allow for a substantial increase of the associative capacity. On the other, the connection between the employed learning principles and the neurological behavior might not necessarily as clear as in model with the Hebbian rule.

\textbf{ \underline{\emph{Correlated/noisy patterns:}}} Majority of the above works consider a common scenario, where the patterns are \emph{independent} among themselves. Such an idealized context does provide a sufficiently simplified model that allows for analytical tractability. Moreover, results obtained for such models provide an intuition that is presumed to be useful for a better understanding of more complex NNs.  In practice the datasets are not as ideal. Instead of being independent the patterns are likely correlated. Such correlations may act as a noise or even portray the presence of some intrinsic features (data redundancy or heterogeneity) that go beyond what can be captured through presumed pattern independence. Along the same lines, if one moves away from the simple pattern retrieval tasks and considers more challenging sequence retrieval, disentanglement, generalization or categorization, then again the idealistic nature of independent patterns  might not suffice \cite{BranchAre92,LeuzziPRT22,Alemanno23,AglABSRF25,Kalaj25}.

Analytical treatments of generic practical correlated scenarios are significantly more challenging. Nonetheless, as understanding of NN's functioning deepened over the last several years, a lot of effort has been put towards achieving further progress.  A particularly successful has been a line of work related to archetype/example patterns. As it is also very closely related to the topics that we study in this paper, we discuss it next in more detail.

\textbf{ \underline{\emph{Archetype/example patterns:}}}  A type of data often used to highlight the correlations effects assumes that to each pattern -- called \emph{archetype} or ground truth -- from a set of independent patterns, one associates its $M$ slightly independently perturbed \emph{example} versions. Perturbations are typically realized by multiplying original pattern elements by independent Bernoulli random variables of mean $r\in[0,1]$ (for $r=1$, one has the (redundant) standard Hopfield-Hebbian model with each pattern repeated $M$ types). When $r<1$, the mathematical problems associated with approximate retrieval are then also tightly connected to generalization, inference, and disentanglement as well \cite{Alemanno23,AglMFV26,AglLF24}. However,  compared to plain (uncorrelated) scenarios, obtaining underlying statistical mechanics pictures is significantly more complicated.

Recently, spectral methods have emerged as particularly successful analytical alternatives to classical replica methods \cite{AglMFV26,AglLF24}. The underlying ideas are truly remarkable and rely on spectral fragmentation (in a way resembling the famous BBP phase transition observed in the so-called spiked models \cite{BBP05}). As system parameters that govern correlations change, the spectrum of particularly designed relevant system matrices changes as well. The changes are of quantitative type until parameters' values reach phase-transition regions, where a qualitative change -- fragmenting of the spectrum from a single bulk structure into a two-pieces formation -- happens. Onset of such a fragmentation is then predicted as a conceivably accurate approximation of the underlying retrieval phase-transition -- in a direct analogy with the appearance of the so-called spectral outliers in BBP transitions considerations \cite{BBP05}. Moreover, not only are these ideas conceptually beautiful, they give excellent practical results as well. In particular, combining spectral and replica theory  \cite{AglMFV26} obtained very elegant single line explicit expression that approximately characterizes impact of correlation parameters $M$ and $r$ on capacity phase-transitions. In case of so-called full rank ($\alpha M>1$) and low correlated (smaller $r$) data, the approximations obtained in \cite{AglMFV26} very closely approach the true performance of the underlying retrieval dynamics. As correlations increase the predictive power decreases and they slowly deviate from the dynamics performance.

\subsection{Our contributions}
\label{sec:contrib}

In this paper we consider the very same \emph{archtypes/examples} set of correlated patterns. We develop Random Duality Theory (RDT) based framework for the analysis of the retrieval (or analogously inference/generalization) dynamics. Utilizing the first lifted level of the framework, we obtain retrieval capacity estimates that are explicit functions of both key correlation parameters $M$ and $r$. For concreteness, we focus on the full rank data (regime where spectral methods tend to show slight deviations from practical dynamics). Our results confirm that spectral predictions from \cite{AglMFV26}  are indeed excellent approximations in the low correlations regimes (precisely as observed in \cite{AglMFV26}). Moreover, we also manage to close the gap for higher correlations. In particular, after practically running the dynamics, we observe that already for fairly small problem dimensions (on the order of $\sim 1000$), simulation results are in an excellent agreement with our theoretical predictions in the entire $r\in[0,1]$ range and for all $M$. Implementing the fully lifted RDT machinery (i.e. conducting the analysis and numerical evaluations on higher lifting levels) would keep closing the gap further. However, keeping in mind that in the ideal scenario (with $r=1$ and $M$ redundant patterns), the capacity estimates obtained on the first ($\approx 0.1379$) and second ($\approx 0.1381$) lifting level differ by $\sim 0.01\%$, no significant corrections on higher lifting levels are expected.

\section{Mathematical preliminaries}
 \label{sec:mathsetup}

We consider standard statistical Hopfield archetype/examples setup. To that end, let $\bg^{(i)}\in \mS_H^n,i=1,2,\dots,m,$ be $m$ independently and uniformly randomly chosen binary patterns ($\mS_H^n$ is the unit $n$-dimensional Hamming sphere, i.e. $\mS_H^n\triangleq\left \{-\frac{1}{\sqrt{n}},\frac{1}{\sqrt{n}}\right \}^n$).  We call these patterns \emph{archetypes} and denote by $\bg_j^{(i)},j=1,2,\dots,n,$ their $j$-th components. For a fixed correlation parameter $r\in[0,1]$, we associate with each $\bg^{(i)}$   $M$ correlated patterns (so-called \emph{examples}), $\bg^{(i,k)}$,  in the following way
\begin{eqnarray}
\bg_j^{(i,k)} = \eta_j^{(k)} \bg_j^{(i)} \quad \mbox{ and } \quad \mP\lp \eta_j^{(k)}=\pm 1\rp =\frac{1\pm r}{2}\quad \mbox{ with }\quad  k=1,2,\dots,M . \label{eq:ex0b0}
\end{eqnarray}
$\eta_j^{(k)}$ are independent among themselves for any choice $j$ and $k$. For any $i$, we also set
\begin{eqnarray}
\bar{\cG}^{(i)} =  \left \{ \bg^{(i,k)} \right \}_{k=1,2,\dots,M} , i=1,2,\dots,m . \label{eq:ex0b1}
\end{eqnarray}
Basically, $\bar{\cG}^{(i)}$ is an array of $M$ patterns associated via correlation (\ref{eq:ex0b0}) with $\bg^{(i)}$. We further set
\begin{eqnarray}
\cG\triangleq \sum_{i=1}^{m}  \sum_{k=1}^{M} \bg^{(i,k)}\lp\bg^{(i,k)}\rp^T, \label{eq:ex0}
\end{eqnarray}
and consider the following Hopfield sequential dynamics
\begin{eqnarray}
\bar{\x}_i(t+1)=\mbox{sign}\lp\sum_{j=1,j\neq i}^{n}\cG_{i,j}\bar{\x}_j(t)\rp. \label{eq:ex0a0}
\end{eqnarray}
Here and throughout the paper, we adopt notational convention that $\cG_{i,j}$ stands for the $ij$-th component of $\cG$ and $\cG_{i,:}$ is the $i$-th row of $\cG$. Defining the energy objective as
\begin{eqnarray}
\cH(\x)\triangleq-\x^T\cG\x, \label{eq:ex0a1}
\end{eqnarray}
we easily have that the above dynamics converges to its a local minimum. Taking $\bar{\x}$ as the dynamics converging point
\begin{eqnarray}
\bar{\x}=\lim_{t\rightarrow \infty} \bar{\x}(t), \label{eq:ex0a2}
\end{eqnarray}
we have for the local energy minimum achieved at  $\bar{\x}$
\begin{eqnarray}
\cH(\bar{\x})=-\bar{\x}^T\cG\bar{\x}. \label{eq:ex0a3}
\end{eqnarray}
Moreover,
\begin{eqnarray}
\cH(\bar{\x})\leq \min_{\|\x-\bar{\x}\|_2=\frac{2}{\sqrt{n}},\x\in S_H^n} -\x^T\cG\x. \label{eq:ex0a4}
\end{eqnarray}
Within the standard Hopfield retrieval, the main idea is that (provided ``sufficient pattern orthogonality'') the converging point may exactly or approximately match one of the stored patterns. In other words, it may happen that $\bar{\x}^T\bg^{(i)}=1$ or $\bar{\x}^T\bg^{(i)}=1-2\delta$, for a small $\delta>0$. The same philosophy applies here as well. The difference though is that this time around the $\bg^{(i)}$ are not the patterns that one can access (it should be noted that they are not parts of $\cG$ and do not participate in dynamics (\ref{eq:ex0a0})). Instead they are the archetypes and one wonders whether these ground truth patterns can be recovered even though they can not be accessed. As we will see, analogously to  \cite{Hop82,AmiGutSom85,MPRV87,FST00,Stojnichebhop24} this will again turn out to be possible even in the linear scaling regime. However, the critical number of patterns for which it will be possible will be a function of underlying correlations parameters $M$ and $r$ that govern to what extent the ground truth archetypes, $\bg^{(i)}$, are hidden within their associated examples groups $\bar{\cG}^{(i)}$.

\subsection{Basin of attraction and dynamics capcity}
 \label{sec:agsbasin}

The basin of attraction that we consider is conceptually similar to AGS \cite{AmiGutSom85}. However, there are also a few technical differences that we highlight below. 

For concreteness and easiness of writing, we take without loss of generality $\bg^{(1)}=\frac{\1}{\sqrt{n}}$ as the targeted archetype pattern. Analogously to the AGS considerations from \cite{AmiGutSom85}, one can define the capacity of the dynamics (\ref{eq:ex0a0}) as the largest $\alpha\triangleq \lim_{n\rightarrow\infty} \frac{m}{n}$ such that for some (small) $\delta>0$
\begin{eqnarray}
\hat{\x}^{(AGS)}(\delta) \triangleq \mbox{arg min}_{\|\x-\bg^{(1)}\|_2^2=4\delta,\x\in S_H^n} -\x^T\cG\x. \label{eq:ex0a5}
\end{eqnarray}
and
\begin{eqnarray}
\cH(\bar{\x}) \leq  \min \lp \cH(\bg^{(1)}), \cH(\hat{\x}^{(AGS)}(\delta))  \rp. \label{eq:ex0a6}
\end{eqnarray}
We define
\begin{eqnarray}
\hat{\xi}_1(\delta) \triangleq \lim_{n\rightarrow\infty}\frac{1}{n}\min_{\|\x-\bg^{(1)}\|_2^2=4\delta,\x\in S_H^n} -\x^T\cG\x, \label{eq:ex0a9}
\end{eqnarray}
and for a $\tilde{\x}$ such that $\tilde{\x}^T\x =1-2\delta$ consider the following approximation
\begin{eqnarray}
\tilde{\xi}_1(\delta) = - \sum_{k=1}^M \tilde{\x}^T\bg^{(1,k)} \lp \bg^{(1,k)}\rp^T\tilde{\x} + \lim_{n\rightarrow\infty}\frac{1}{n}
\min_{\|\x-\bg^{(1)}\|_2^2=4\delta,\x\in S_H^n} -\sum_{i=2}^{m}\sum_{k=1}^{M}\x^T\bg^{(i,k)} \lp \bg^{(i,k)}\rp^T\x. \label{eq:ex0a10}
\end{eqnarray}
Setting
\begin{eqnarray}
 \bar{G}^{(i)}\triangleq \begin{bmatrix}
               \lp\bg^{(i,1)}\rp^T \\ \lp\bg^{(i,2)}\rp^T \\ \vdots \\ \lp\bg^{(i,M)}\rp^T
             \end{bmatrix}
\quad \mbox{ and } \quad
 \bar{G}\triangleq \begin{bmatrix}
               \bar{G}^{(2)} \\ \bar{G}^{(3)}  \\ \vdots \\ \bar{G}^{(m)} 
             \end{bmatrix}, \label{eq:ex0a11}
\end{eqnarray}
allows to rewrite (\ref{eq:ex0a10}) as
\begin{eqnarray}
\tilde{\xi}_1(\delta)
& = & - \sum_{k=1}^M \tilde{\x}^T\bg^{(1,k)} \lp \bg^{(1,k)}\rp^T\tilde{\x} + \lim_{n\rightarrow\infty} \frac{1}{n}
\min_{\|\x-\bg^{(1)}\|_2^2=4\delta,\x\in S_H^n} - \x^T\bar{G}^T\bar{G}\x \nonumber \\
& = & - \sum_{k=1}^M \tilde{\x}^T\bg^{(1,k)} \lp \bg^{(1,k)}\rp^T\tilde{\x} + \lim_{n\rightarrow\infty} \frac{1}{n}
\min_{\|\x-\bg^{(1)}\|_2^2=4\delta,\x\in S_H^n} - \|\bar{G}\x\|_2^2 \nonumber \\
& = & - \sum_{k=1}^M \tilde{\x}^T\bg^{(1,k)} \lp \bg^{(1,k)}\rp^T\tilde{\x} 
- \lim_{n\rightarrow\infty} \frac{1}{n} \max_{\|\x-\bg^{(1)}\|_2^2=4\delta,\x\in S_H^n}  \|\bar{G}\x\|_2^2 \nonumber \\
& = & - \sum_{k=1}^M \tilde{\x}^T\bg^{(1,k)} \lp \bg^{(1,k)}\rp^T\tilde{\x} -\bar{\xi}(\delta)^2, \label{eq:ex0a12}
\end{eqnarray}
where
\begin{eqnarray}
\bar{\xi}(\delta) \triangleq \lim_{n\rightarrow\infty} \frac{1}{\sqrt{n}} \max_{\|\x-\bg^{(1)}\|_2^2=4\delta,\x\in S_H^n}  \|\bar{G}\x\|_2. \label{eq:ex0a13}
\end{eqnarray}
As all relevant quantities will concentrate, we find it useful to note
\begin{eqnarray}
\mE \,\tilde{\xi}_1(\delta)
 & = & - \mE\, \sum_{k=1}^M \tilde{\x}^T\bg^{(1,k)} \lp \bg^{(1,k)}\rp^T\tilde{\x} - \mE \, \bar{\xi}(\delta)^2
 = -(1-2\delta)^2r^2 M  - \mE \, \bar{\xi}(\delta)^2. \label{eq:ex0a12b0}
\end{eqnarray}
The optimization in (\ref{eq:ex0a13}) can be reformulated as
\begin{eqnarray}
\bar{\xi}(\delta) \triangleq \lim_{n\rightarrow\infty}\frac{1}{\sqrt{n}} \max_{\x\in \cX(\delta)}\max_{\bar{\y}\in\mS^{mM}}  \bar{\y}^T\bar{G}\x, \label{eq:ex0a14}
\end{eqnarray}
where $\mS^{mM}$ is the unit $mM$-dimensional Euclidean sphere, and for $\delta\in [0,1]$
\begin{eqnarray}
\cX(\delta) \triangleq \left \{\x \,\, | \quad \|\x-\bg^{(1)}\|_2^2=4\delta,\x\in S_H^n\right \}. \label{eq:ex0a15}
\end{eqnarray}
To make the presentation easier, we switch from Bernoulli to Gaussian variables (concentrations and the Lindeberg variant of the central limit theorem (see, e.g., \cite{Chatterjee06}, for its a particularly elegant implementation)  allow for such universality switch).  However, one also has to keep in mind that the correlations within $\bar{G}^{(i)}$ need to be preserved. To that end, we consider the following Gaussianized variant of (\ref{eq:ex0a14}) 
\begin{eqnarray}
\xi(\delta) \triangleq \lim_{n\rightarrow\infty}\frac{1}{\sqrt{n}} \max_{\x\in \cX(\delta)}\max_{\bar{\y}\in\mS^{mM}}  \bar{\y}^T  \sqrt{I_{m\times m}\otimes ( (1-r^2)I_{M\times M} + r^2 \1_M\1_M^T )}  G\x, \label{eq:ex0a14b0}
\end{eqnarray}
where $G\in\mR^{mM\times n}$ is comprised of independent standard normals and, analogously to (\ref{eq:ex0a11}), can be written in the following block form
\begin{eqnarray}
 G^{(i)}\triangleq \begin{bmatrix}
               \lp\g^{(i,1)}\rp^T \\ \lp\g^{(i,2)}\rp^T \\ \vdots \\ \lp\g^{(i,M)}\rp^T
             \end{bmatrix}
\quad \mbox{ and } \quad
 G\triangleq \begin{bmatrix}
               G^{(2)} \\ G^{(3)}  \\ \vdots \\ G^{(m)} 
             \end{bmatrix}. \label{eq:ex0a14b1}
\end{eqnarray}
In (\ref{eq:ex0a14b0}), $I$ and $\1$ stand for the identity matrix and column vector of all ones, respectively. Their dimensions are specified in the subscripts (or omitted if clear form the context). We also observe the following alterative formulation of (\ref{eq:ex0a14b0})
\begin{eqnarray}
\xi(\delta) \triangleq \lim_{n\rightarrow\infty}\frac{1}{\sqrt{n}} \max_{\x\in \cX(\delta)}\max_{\bar{\y}\in\mS^{mM}}  \bar{\y}^T  \tilde{G}\x, \label{eq:ex0a14b2}
\end{eqnarray}
where
\begin{eqnarray}
  \tilde{G}\triangleq \begin{bmatrix}
               \tilde{G}^{(2)} \\ \tilde{G}^{(3)}  \\ \vdots \\ \tilde{G}^{(m)} 
             \end{bmatrix}
             \quad \mbox{ and } \quad
\tilde{G}^{(i)} = \sqrt{(1-r^2)I_{M\times M} + r^2 \1_M\1_M^T} G^{(i)}.
              \label{eq:ex0a14b3}
\end{eqnarray}
To see that (\ref{eq:ex0a14b2}) preserves all key statistical properties of $\bar{G}$ in (\ref{eq:ex0a14}), we consider the $i$-th block $\bar{G}^{(i)}$ and its analogue $\tilde{G}^{(i)} = \sqrt{(1-r^2)I_{M\times M} + r^2 \1_M\1_M^T} G^{(i)}$. Since both $\bar{G}$ and $G$ have independent columns, we have for the $j$-th column  
\begin{eqnarray}
\mE\, \bar{G}_{:,j}^{(i)} \lp\bar{G}_{:,j}^{(i)}\rp^T & = & (1-r^2)I_{M\times M} + r^2 \1_M\1_M^T 
\nonumber \\
& = & \sqrt{(1-r^2)I_{M\times M} + r^2 \1_M\1_M^T} \sqrt{(1-r^2)I_{M\times M} + r^2 \1_M\1_M^T}
\nonumber \\
& = & 
 \mE\, \sqrt{(1-r^2)I_{M\times M} + r^2 \1_M\1_M^T}  G_{:,j}^{(i)} \lp G_{:,j}^{(i)}\rp^T\sqrt{(1-r^2)I_{M\times M} + r^2 \1_M\1_M^T}  
 \nonumber \\
& = & 
 \mE\,  \tilde{G}_{:,j}^{(i)} \lp \tilde{G}_{:,j}^{(i)}\rp^T.  
 \label{eq:ex0a14b4}
\end{eqnarray}
Analogously to (\ref{eq:ex0a12b0}), we then set
\begin{eqnarray}
  \xi_1(\delta)
  = -(1-2\delta)^2r^2 M  - \xi(\delta)^2, \label{eq:ex0a14b5}
\end{eqnarray}
and for any $M\in\mN$ and $r\in[0,1]$ define the following as the estimate of the critical Hopfield-Hebbian dynamics (\ref{eq:ex0a1}) capacity
 \begin{equation}
\alpha_c(M,r)  \triangleq  \max \{\alpha |\hspace{.08in} \exists \delta, \,\, \delta_u>\delta >\delta_1>0, \lim_{n\rightarrow\infty}
\mP_G\lp \xi_1(\delta_1)\leq \min( \xi_1(0),\xi_1(\delta))\rp\longrightarrow 1\}
\quad \mbox{ with } \quad \alpha  = \lim_{n\rightarrow \infty} \frac{m}{n}, 
  \label{eq:ex0a16}
\end{equation}
where $\delta_u$ is the predetermined pattern neighborhood. Practically speaking, $\alpha_c(M,r)$ relates to the largest number of patterns for which there is a local minima of the dynamics in (\ref{eq:ex0a1}) within (small) distance $\delta_u$.

\subsection{Free energy formulation}
 \label{sec:feeeng}

As is clear from (\ref{eq:ex0a16})), understanding $\xi_1(\delta)$ and $\xi(\delta)$ is critically important for finding $\alpha_c(M,r)$. We therefore focus on (\ref{eq:ex0a14b0}). Analogously to
\cite{StojnicGardGen13}, we recognize that (\ref{eq:ex0a14b0}), can be directly connected to studying statistical physics \emph{free energies}. We sketch the key contours of the correspondence and refer to \cite{StojnicGardGen13} for a more detailed discussion.

From (\ref{eq:ex0a14b0}) we have
\begin{eqnarray}
\xi(\delta) = - \lim_{n\rightarrow\infty}\frac{1}{\sqrt{n}} 
\min_{\x\in \cX(\delta),\z\in\mR^{mM}} 
& &  -\sqrt{\z^T  I_{m\times m}\otimes ( (1-r^2)I_{M\times M} + r^2 \1_M\1_M^T ) \z}
\nonumber \\
\mbox{subject to } & &  G\x =\z ,
 \label{eq:ham0a0}
\end{eqnarray}
which is further equivalent to 
\begin{align}
\xi(\delta) & =  - \lim_{n\rightarrow\infty}\frac{1}{\sqrt{n}} 
\min_{\x\in \cX(\delta),\z\in\mR^{mM}} \max_{\y\in\mR^{mM}}
 \lp -\sqrt{\z^T  I_{m\times m}\otimes ( (1-r^2)I_{M\times M} + r^2 \1_M\1_M^T ) \z}
+\y^T G\x -\y^T\z \rp
\nonumber \\
& =  - \lim_{n\rightarrow\infty}\frac{1}{\sqrt{n}} 
\min_{\x\in \cX(\delta),\z\in\mR^{mM}} \max_{y>0,\y\in y\mS^{mM}}
 \lp -\sqrt{\z^T  I_{m\times m}\otimes ( (1-r^2)I_{M\times M} + r^2 \1_M\1_M^T ) \z}
+\y^T G\x -\y^T\z \rp.\nonumber \\
 \label{eq:ham0a1}
\end{align}
We consider the following Hamiltonian 
\begin{equation}
\cH_{sq}(G)= \y^TG\x,\label{eq:ham1}
\end{equation}
and associated partition function
\begin{equation}
Z_{sq}(\beta,G)=\sum_{\x\in\cX,\z\in\cZ} \lp \sum_{\y\in\cY}e^{\beta\lp \cH_{sq}(G) +f(\y;\z) \rp } \rp^{-1}.  \label{eq:partfun}
\end{equation}
The following is then thermodynamic limit of the average free energy 
\begin{eqnarray}
f_{sq}(\beta) & = & \lim_{n\rightarrow\infty}\frac{\mE_G\log{(Z_{sq}(\beta,G)})}{\beta \sqrt{n}}
=\lim_{n\rightarrow\infty} \frac{\mE_G\log\lp \sum_{\x\in\cX,\z\in\cZ} \lp \sum_{\y\in\cY}e^{\beta \lp \cH_{sq}(G) +f(\y;\z) \rp  } \rp^{-1} \rp}{\beta \sqrt{n}} \nonumber \\
& = &\lim_{n\rightarrow\infty} \frac{\mE_G\log\lp \sum_{\x\in\cX,\z\in\cZ} \lp  \sum_{\y\in\cY}e^{\beta(\y^TG\x+f(\y;\z))  }\rp^{-1} \rp}{\beta \sqrt{n}}.\label{eq:logpartfunsqrt}
\end{eqnarray}
Throughout the paper, if the subscripts next to $\mP$ and $\mE$ are present they denote the underlying randomness (when the randomness is clear from the context, the subscripts are often left unspecified).

For sets specializations $\cX=\cX(\delta)$ and $\cY=y\mS^{mM}$, the ground state energy (obtained in the zero-temperature $T\rightarrow 0$ or  $\beta=\frac{1}{T}\rightarrow\infty$ regime)  
\begin{eqnarray}
f_{sq}(\infty)  & \triangleq  &  \lim_{\beta\rightarrow\infty}f_{sq}(\beta)  = 
\lim_{\beta,n\rightarrow\infty}\frac{\mE_G\log{(Z_{sq}(\beta,G)})}{\beta \sqrt{n}}
\nonumber \\
& = &
 \lim_{n\rightarrow\infty}\frac{\mE_G \max_{\x\in\cX,\z\in\cZ}  - \max_{\y\in\cY} \lp\y^TG\x +f(\y;\z)\rp }{\sqrt{n}}
 \nonumber \\
& = &
- \lim_{n\rightarrow\infty}\frac{\mE_G \min_{\x\in\cX,\z\in\cZ} \max_{\y\in\cY} \lp\y^TG\x +f(\y;\z)\rp }{\sqrt{n}},
  \label{eq:limlogpartfunsqrta0}
\end{eqnarray}
is directly related to (\ref{eq:ham0a1})  (and consequently to (\ref{eq:ex0a14b0})). Moreover, $f_{sq}(\infty)$ is then tightly connected to $\xi_1(\delta)$ as well, which further means that it is critically important for finding $\alpha_c(M,r)$ in  (\ref{eq:ex0a16}). However, determining $f_{sq}(\infty)$ directly is typically not overly trivial. One instead often studies  $f_{sq}(\beta)$ for a general $\beta$  and then specializes to $\beta\rightarrow\infty$. As such a specialization is the ultimate goal, throughout the analysis that follows below we may sometimes neglect details without significant contribution in the $\beta\rightarrow\infty$ regime.

\section{Fitting Hopfield-Hebbian dynamics capacity within RDT}
\label{sec:randlincons}

Observation that the free energy from (\ref{eq:logpartfunsqrt}) is a function of \emph{bilinearly indexed} (bli) random process $\y^TG\x$ allows utilization of the bli related results from \cite{Stojnicsflgscompyx23,Stojnicnflgscompyx23,Stojnicflrdt23}. For a practical implementation, we follow into the footsteps of \cite{Stojnicalgbp25,Stojnicalgsbp26} and start with necessary technical preliminaries. 

Consider sets $\cX\subseteq \mS^n$ and $\cY\subseteq y \mS^{mM}$ (where $\mS^p$ is the unit sphere in $\mR^p$) and function $f_S(\cdot):\mR^{n+m}\rightarrow R$. For $r\in\mN$, let  $\p=[\p_0,\p_1,\dots,\p_{r+1}]$, $\q=[\q_0,\q_1,\dots,\q_{r+1}]$, and $\c=[\c_0,\c_1,\dots,\c_{r+1}]$ be such that
 \begin{eqnarray}\label{eq:hmsfl2}
1=\p_0\geq \p_1\geq \p_2\geq \dots \geq \p_r\geq \p_{r+1} & = & 0 \nonumber \\
1=\q_0\geq \q_1\geq \q_2\geq \dots \geq \q_r\geq \q_{r+1} & = &  0,
 \end{eqnarray}
$\c_0=1$, $\c_{r+1}=0$. For $k\in\{1,2,\dots,r+1\}$, let ${\mathcal U}_k\triangleq [u^{(4,k)},\u^{(2,k)},\h^{(k)}]$ be such that the elements of  $u^{(4,k)}\in\mR$, $\u^{(2,k)}\in\mR^{mM}$, and $\h^{(k)}\in\mR^n$ are independent standard normals. We set
  \begin{eqnarray}\label{eq:fl4}
\bar{\psi}_{S,\infty}(f_{S},\cX,\cY,\p,\q,\c,s)  =
 \mE_{G,{\mathcal U}_{r+1}} \frac{1}{n\c_r} \log
\lp \mE_{{\mathcal U}_{r}} \lp \dots \lp \mE_{{\mathcal U}_3}\lp\lp\mE_{{\mathcal U}_2} \lp \lp Z_{S,\infty}\rp^{\c_2}\rp\rp^{\frac{\c_3}{\c_2}}\rp\rp^{\frac{\c_4}{\c_3}} \dots \rp^{\frac{\c_{r}}{\c_{r-1}}}\rp, \nonumber \\
 \end{eqnarray}
where
\begin{eqnarray}\label{eq:fl5}
Z_{S,\infty} & \triangleq & e^{D_{0,S,\infty}} \nonumber \\
 D_{0,S,\infty} & \triangleq  & \max_{\x\in\cX } s \max_{\y\in\cY }
 \lp \sqrt{n} f_{S}
+  \sqrt{n}  y    \lp\sum_{k=2}^{r+1}c_k\h^{(k)}\rp^T\x
+ \sqrt{n}  \y^T\lp\sum_{k=2}^{r+1}b_k\u^{(2,k)}\rp \rp \nonumber  \\
 b_k & \triangleq & b_k(\p,\q)=\sqrt{\p_{k-1}-\p_k} \nonumber \\
c_k & \triangleq & c_k(\p,\q)=\sqrt{\q_{k-1}-\q_k},
 \end{eqnarray}
and $s^2=1$ (we note that $\bar{\psi}_{S,\infty}(f_{S},\cX,\cY,\p,\q,\c,s) $ form in (\ref{eq:fl4}) is practically useful only for $r\geq 2$. For the notational easiness, we add that when $r=1$, $\bar{\psi}_{S,\infty}(f_{S},\cX,\cY,\p,\q,\c,s) $ is obtained from (\ref{eq:fl4})  for $r=2$ under the $\lim_{\p_2,\q_2,c_2\rightarrow 0}$). The above definitions allow to recall the following fundamental fl-RDT theorem.

\begin{theorem} \cite{Stojnicflrdt23,Stojnicalgbp25,Stojnicalgsbp26}
\label{thm:thmsflrdt1}  Let  $\alpha=\lim_{n\rightarrow\infty} \frac{m}{n}$ remain constant as  $n$ grows. Consider $G\in\mR^{mM \times n}$ with independent standard normal elements and for $y>0$ let $\cX\subseteq \mS^n$ and $\cY\subseteq y\mS^{mM}$ be given (scaled) unit spheres subsets.  Assume that  (\ref{eq:hmsfl2})--({\ref{eq:fl5}}) hold for an $r\in \mN$ (possibly for $r\rightarrow \infty$ as well). Consider the complete stationarized fl-RDT frame from \cite{Stojnicflrdt23,Stojnicalgbp25,Stojnicalgsbp26} and for a given function $f(\y):R^{mM}\rightarrow R$ set
\begin{align}\label{eq:thmsflrdt2eq1a0}
   \psi_{rp} & \triangleq - \max_{\x\in\cX} s \max_{\y\in\cY} \lp \y^TG\x + f(\y) \rp
   \qquad  \mbox{(\bl{\textbf{random primal}})} \nonumber \\
   \psi_{rd}(\p,\q,\c) & \triangleq    \frac{y^2}{2}    \sum_{k=2}^{r+1}\Bigg(\Bigg.
   \p_{k-1}\q_{k-1}
   -\p_{k}\q_{k}
  \Bigg.\Bigg)
\c_k
  - \bar{\psi}_{S,\infty}(f(\y),\cX,\cY,\p,\q,\c,s) \quad \mbox{(\bl{\textbf{fl random dual}})}. \nonumber \\
 \end{align}
Let $\hat{\p_0}\rightarrow 1$, $\hat{\q_0}\rightarrow 1$, and $\hat{\c_0}\rightarrow 1$, $\hat{\p}_{r+1}=\hat{\q}_{r+1}=\hat{\c}_{r+1}=0$ and  let the non-fixed parts of $\hat{\p}$, $\hat{\q}$, and  $\hat{\c}$ be the solutions of the following system
\begin{eqnarray}\label{eq:thmsflrdt2eq2a0}
   \frac{d \psi_{rd}(\p,\q,\c)}{d\p} =  0,\quad
   \frac{d \psi_{rd}(\p,\q,\c)}{d\q} =  0,\quad
   \frac{d \psi_{rd}(\p,\q,\c)}{d\c} =  0.
 \end{eqnarray}
 Then,
\begin{eqnarray}\label{eq:thmsflrdt2eq3a0}
    \lim_{n\rightarrow\infty} \frac{\mE_G  \psi_{rp}}{\sqrt{n}}
  & = &
 \lim_{n,r\rightarrow\infty} \psi_{rd}(\hat{\p},\hat{\q},\hat{\c}) \qquad \mbox{(\bl{\textbf{strong sfl random duality}})},\nonumber \\
 \end{eqnarray}
where $\bar{\psi}_{S,\infty}(\cdot)$ is as in (\ref{eq:fl4})-(\ref{eq:fl5}).
 \end{theorem}
 
The readers familiar with Parisi's RSB ansatz and replica theory in general \cite{Parisi80,Par79,Par80,Par83} may note the presence of added Gaussian sequences nesting averaging in the above theorem. While replicated systems utilized to create  \emph{index decoupling}  are not present in (\ref{eq:hmsfl2})--({\ref{eq:thmsflrdt2eq3a0}), a bit of resembling decoupling flair may be sensed through the presence of bilinearly coupled $f_{sq}(\cdot)$ and fully decoupled $\bar{\psi}_{S,\infty}(\cdot)$. However, the mathematical machinery underlying the above theorem leans on nested random processes studies and in no way relates to any form of replication (for more on core RDT intricacies as well as on the associated algorithmic and machine learning implications, see \cite{Stojnicsgenprovrsb23,Stojnicnflgscompyx23,StojnicRegRndDlt10,StojnicGardGen13,Stojnicalgbp25,Stojnicclupsk25,Stojnicalgsbp26,StojnicNN27,Stojnictcmspnncaprdt23,Stojnictcmspnncapliftedrdt23}).

The key benefit that Theorem \ref{thm:thmsflrdt1} offers is the \emph{random dual}'s simplified structure. However, its ultimate practical relevance relies on one's ability to conduct underlying numerical evaluations. Two types of obstacles usually appear in this regard: (\textbf{\emph{i}}) sets $\cX$ and $\cY$  may lack component-wise structure which renders decoupling of $\x$ and/or $\y$ as not necessarily trivial; and  (\textbf{\emph{ii}}) since \emph{a priori}  $r$  can be arbitrarily large, the associated $r$ convergence of the program outlined om Theorem \ref{thm:thmsflrdt1}  may not be sufficiently fast to make higher lifting levels' numerical evaluations feasible. Drawing analogy with Hopfield-Hebbian dynamics of uncorrelated patterns, the impact of the latter is unlikely to be of relevance. However, the former does pose a challenge and we address it below.

\subsection{Functional and set specializations}
\label{sec:spec}

Theorem \ref{thm:thmsflrdt1} is a generic result.  To make it practically operational in the context of our interest here, we introduce in the following corollaries a couple of specializations.

\begin{corollary}
\label{cor:corsflrdt1}  Assume the setup of Theorem \ref{thm:thmsflrdt1}. For $M\in\mN$ and $r\in[0,1]$ set 
\begin{equation}\label{eq:corthmsflrdt2eq1a0b0}
f(\y)= -\sqrt{\z^T  I_{m\times m}\otimes ( (1-r^2)I_{M\times M} + r^2 \1_M\1_M^T ) \z
  -\y^T\z}, 
  \end{equation}
  and \begin{align}\label{eq:corthmsflrdt2eq1a0}
   \psi_{rp} & \triangleq - \max_{\x\in\cX,\z\in\mR^{mM}} s \max_{\y\in\cY} \lp \y^TG\x + f(\y) \rp
   \qquad  \mbox{(\bl{\textbf{random primal}})} \nonumber \\
   \psi_{rd}(\p,\q,\c) & \triangleq    \frac{y^2}{2}    \sum_{k=2}^{r+1}\Bigg(\Bigg.
   \p_{k-1}\q_{k-1}
   -\p_{k}\q_{k}
  \Bigg.\Bigg)
\c_k
  - \bar{\psi}_{S,\infty}(f(\y),\cX,\cY,\p,\q,\c,s) \quad \mbox{(\bl{\textbf{fl random dual}})}. \nonumber \\
 \end{align}
Let $\hat{\p_0}\rightarrow 1$, $\hat{\q_0}\rightarrow 1$, and $\hat{\c_0}\rightarrow 1$, $\hat{\p}_{r+1}=\hat{\q}_{r+1}=\hat{\c}_{r+1}=0$ and  let the non-fixed parts of $\hat{\p}$, $\hat{\q}$, and  $\hat{\c}$ be the solutions of the following system
\begin{eqnarray}\label{eq:corthmsflrdt2eq2a0}
   \frac{d \psi_{rd}(\p,\q,\c)}{d\p} =  0,\quad
   \frac{d \psi_{rd}(\p,\q,\c)}{d\q} =  0,\quad
   \frac{d \psi_{rd}(\p,\q,\c)}{d\c} =  0.
 \end{eqnarray}
 Then,
\begin{eqnarray}\label{eq:corthmsflrdt2eq3a0}
    \lim_{n\rightarrow\infty} \frac{\mE_G  \psi_{rp}}{\sqrt{n}}
  & = &
 \lim_{n,r\rightarrow\infty} \psi_{rd}(\hat{\p},\hat{\q},\hat{\c}) \qquad \mbox{(\bl{\textbf{strong sfl random duality}})},
 \end{eqnarray}
where $\bar{\psi}_{S,\infty}(\cdot)$ is as in (\ref{eq:fl4})-(\ref{eq:fl5}) and
\begin{eqnarray}\label{eq:corfl5}
  D_{0,S,\infty} & = & \max_{\x\in\cX , \z\in\mR^{mM}} s \max_{\y\in\cY }
 \lp \sqrt{n} f(\y)
+\sqrt{n}    y  \lp\sum_{k=2}^{r+1}c_k\h^{(k)}\rp^T\x
+ \sqrt{n}  \y^T\lp\sum_{k=2}^{r+1}b_k\u^{(2,k)}\rp \rp  .
 \end{eqnarray}
 \end{corollary}
\begin{proof}
Follows automatically from Theorem \ref{thm:thmsflrdt1} after choosing $f(\y)$ as $f(\y;\z)$ in (\ref{eq:corthmsflrdt2eq1a0b0}) and trivially adjusting for an additional $\z$ maximization.
 \end{proof}

In addition to the above functional specialization, we also consider set specialization 
$\cX=\cX(\delta)$. To that end, we first define
  \begin{eqnarray}\label{eq:prac2}
\varphi(D,\c) & = &
 \mE_{G,{\mathcal U}_{r+1}} \frac{1}{\c_r} \log
\lp \mE_{{\mathcal U}_{r}} \lp \dots \lp \mE_{{\mathcal U}_3}\lp\lp\mE_{{\mathcal U}_2} \lp
\lp    e^{D}   \rp^{\c_2}\rp\rp^{\frac{\c_3}{\c_2}}\rp\rp^{\frac{\c_4}{\c_3}} \dots \rp^{\frac{\c_{r}}{\c_{r-1}}}\rp,  
  \end{eqnarray}
\begin{eqnarray}\label{eq:prac4}
D^{(per)}_i(c_k) & = & -2\min \lp\sum_{k=2}^{r+1}c_k\h_i^{(k)}-\nu,0\rp +\sum_{k=2}^{r+1}c_k\h_i^{(k)},
\nonumber \\
   D_{\nu}^{(tot)} & =  &  
     -  \min_{\z\in\mR^n} 
\Bigg (\Bigg . -\sqrt{n}\sqrt{\z^T I_{m\times m}\otimes ((1-r^2)I_{M\times M} +r^2\1_M\1_M^T)\z}
+ \sqrt{n} y \Bigg (\Bigg .
 \left \| \sum_{k=2}^{r+1}b_k\u^{(2,k)}  - \z \right \|_2  
\nonumber \\
& & 
- \lp  -2\nu\delta n  + \sum_{i=1}^nD^{(per)}_i(c_k) \rp \Bigg . \Bigg )
\Bigg . \Bigg )
 \end{eqnarray}
 and
\begin{eqnarray}
  \frac{y^2}{2}    \sum_{k=2}^{r+1}\Bigg(\Bigg.
    \p_{k-1} \q_{k-1}
   -\p_{k} \q_{k}
  \Bigg.\Bigg)
\hat{\c}_k
  - \frac{1}{n}\varphi(D_{\nu}^{(tot)},\c)  
 \triangleq    \bar{\psi}_{rd}(\p,\q,\c,\nu,y) .
  \label{eq:negprac13}
\end{eqnarray}

\begin{corollary}
\label{cor:corsflrdt2} 
  Assume the setup of Corollary \ref{cor:corsflrdt1} with $s=-1$ and $\cX=\cX(\delta)$, where for $\delta>0$, $\cX(\delta) $ is as in (\ref{eq:ex0a15}). For $\varphi(\cdot)$ as in (\ref{eq:prac2}) and $\bar{\psi}(\cdot)$ as in (\ref{eq:negprac13}), let the ``fixed'' parts of $\hat{\p}$, $\hat{\q}$, and $\hat{\c}$ satisfy $\hat{\p}_1\rightarrow 1$, $\hat{\q}_1\rightarrow 1$, $\hat{\c}_1\rightarrow 1$, $\hat{\p}_{r+1}=\hat{\q}_{r+1}=\hat{\c}_{r+1}=0$ and let the ``non-fixed'' parts of $\hat{\p}_k$, $\hat{\q}_k$, and $\hat{\c}_k$ ($k\in\{2,3,\dots,r\}$) satisfy
  \begin{eqnarray}\label{eq:negthmprac1eq1}
   \frac{d \bar{\psi}_{rd}(\p,\q,\c,\nu,y)}{d\p} =
   \frac{d \bar{\psi}_{rd}(\p,\q,\c,\nu,y)}{d\q} =
   \frac{d \bar{\psi}_{rd}(\p,\q,\c,\nu,y)}{d\c} =  
   \frac{d \bar{\psi}_{rd}(\p,\q,\c,\nu,y)}{d\nu} = 
   \frac{d \bar{\psi}_{rd}(\p,\q,\c,\nu,y)}{dy} =0.\nonumber \\
 \end{eqnarray}
One then has
 \begin{eqnarray}
 \lim_{n\rightarrow\infty} \psi_{rd}(\hat{\p},\hat{\q},\hat{\c})
     & = & \lim_{n\rightarrow \infty} \frac{y^2}{2}    \sum_{k=2}^{r+1}\Bigg(\Bigg.
   \hat{\p}_{k-1}\hat{\q}_{k-1}
   -\hat{\p}_{k}\hat{\q}_{k}
  \Bigg.\Bigg)
\hat{\c}_k
  - \frac{1}{n}\varphi(D^{(tot)},\c)  
\nonumber \\
& = &  \lim_{n\rightarrow\infty}  \min_{\nu} \max_{y>0} \bar{\psi}_{rd}(\hat{\p},\hat{\q},\hat{\c},\nu,y) , 
   \label{eq:negthmprac1eq2}
\end{eqnarray}
with $c_k(\hat{\p},\hat{\q})   =  \sqrt{\hat{\q}_{k-1}-\hat{\q}_k}$ and $b_k(\hat{\p},\hat{\q})   =  \sqrt{\hat{\p}_{k-1}-\hat{\p}_k}$.
\end{corollary}
\begin{proof}
See Appendix \ref{sec:app0}.
\end{proof}

\subsection{Concrete implementation -- first level of lifting}
\label{sec:firstlev}

For the first level, we have $r=1$ and $\hat{\p}_1\rightarrow 1$ and $\hat{\q}_1\rightarrow 1$ which, together with $\hat{\p}_{r+1}=\hat{\p}_{2}=\hat{\q}_{r+1}=\hat{\q}_{2}=0$, and $\hat{\c}_{2}\rightarrow 0$, gives
\begin{align}\label{eq:negprac19}
    \bar{\psi}_{rd}^{(1)}(\hat{\p},\hat{\q},\hat{\c},\nu,y)   & =  - \frac{y^2}{2}
\hat{\c}_2
- 2y\nu\delta + \frac{1}{\hat{\c}_2}\log\lp \mE_{{\mathcal U}_2} e^{\hat{\c}_2 y  \lp  -2\min  \lp \sqrt{1-0}\h_1^{(2)}-\nu,0 \rp +\sqrt{1-0}\h_1^{(2)} \rp}\rp
\nonumber \\
& \quad  + \frac{1}{\hat{n\c}_2}\log\lp \mE_{{\mathcal U}_2} e^{\hat{\c}_2 D_{\nu}^{(tot)}  } \rp \nonumber \\
& \rightarrow
- 2y\nu\delta + \frac{1}{\hat{\c}_2}\log\lp 1+ \mE_{{\mathcal U}_2} \hat{\c}_2 y \lp  -2\min  \lp \sqrt{1-0}\h_1^{(2)}-\nu,0 \rp +\sqrt{1-0}\h_1^{(2)} \rp \rp
 \nonumber \\
&  \quad  
+ \frac{1}{n\hat{\c}_2}\log\lp 1 + \mE_{{\mathcal U}_2} \hat{\c}_2D_{\nu}^{(tot)}\rp \nonumber \\
 & \rightarrow
- 2y\nu\delta + y\mE_{{\mathcal U}_2}  \lp  -2\min  \lp \h_1^{(2)}-\nu,0 \rp +\h_1^{(2)} \rp
+ \frac{1}{n}\mE_{{\mathcal U}_2} D_{\nu}^{(tot)}
 \nonumber \\
 & \rightarrow
- 2y\nu\delta + y\mE_{{\mathcal U}_2}  \lp  -2\min  \lp \h_1^{(2)}-\nu,0 \rp +\h_1^{(2)} \rp
\nonumber \\
& \quad
 - \frac{1}{\sqrt{n}} \mE_{{\mathcal U}_2}  \min_{\z\in\mR^n} 
\Bigg (\Bigg . -\sqrt{\z^T I_{m\times m}\otimes ((1-r^2)I_{M\times M} +r^2\1_M\1_M^T)\z}
+   y 
 \left \| \sum_{k=2}^{2}b_2\u^{(2,2)}  - \z \right \|_2  \Bigg .\Bigg )
 \nonumber \\
 & \rightarrow
  -  \mE_{{\mathcal U}_2} \min_{\z\in\mR^n} 
\Bigg (\Bigg . -\frac{1}{\sqrt{n}}\sqrt{\z^T I_{m\times m}\otimes ((1-r^2)I_{M\times M} +r^2\1_M\1_M^T)\z}
\nonumber \\
& \quad +   y \lp 
 \frac{1}{\sqrt{n}}\left \| \sqrt{1-0}\u^{(2,2)}  - \z \right \|_2 -f_1(\nu) \rp \Bigg .\Bigg )
,
  \end{align}
 where
\begin{eqnarray}\label{eq:negprac19a1}
f_1(\nu) \triangleq \mE_{{\mathcal U}_2}  \lp  2\max  \lp -\h_1^{(2)}+\nu,0 \rp +\h_1^{(2)} \rp - 2\nu\delta
=\mE_{{\mathcal U}_2}  \lp  2\max  \lp -\h_1^{(2)}+\nu,0 \rp \rp - 2\nu\delta .
  \end{eqnarray}
After solving the integrals one obtains
\begin{eqnarray}\label{eq:negprac19a2}
f_1(\nu) & = &   2\lp \frac{1}{\sqrt{2\pi}} e^{-\nu^2/2}+\frac{\nu}{2}\erfc\lp - \frac{\nu}{\sqrt{2}}\rp\rp-2\nu\delta.
  \end{eqnarray}
Taking the $\nu$ derivative gives 
\begin{eqnarray}\label{eq:negprac19a3}
\frac{df_1(\nu)}{d\nu}
& = & 2\lp \frac{1}{2}\erfc\lp -\frac{\nu}{\sqrt{2}}\rp \rp-2\delta  = \erfc \lp -\frac{\nu}{\sqrt{2}}\rp -2\delta
 =1-2\delta - \erf \lp -\frac{\nu}{\sqrt{2}}\rp .
  \end{eqnarray}
Equalling the above derivative to zero, one also finds
$\hat{\nu}=-\sqrt{2} \erfinv \lp 1-2\delta\rp$, which then gives
\begin{eqnarray}\label{eq:negprac19a4}
f_1(\hat{\nu})
& = & 2\lp \frac{1}{\sqrt{2\pi}} e^{-\frac{\hat{\nu}^2}{2}}+\frac{\hat{\nu}}{2}\erfc\lp - \frac{\hat{\nu}}{\sqrt{2}}\rp\rp-2\hat{\nu}\delta
=
 \frac{2}{\sqrt{2\pi}} e^{-\frac{\hat{\nu}^2}{2}}+\hat{\nu}\erfc\lp - \frac{\hat{\nu}}{\sqrt{2}}\rp - 2\hat{\nu}\delta \nonumber \\
& = &
 \frac{2}{\sqrt{2\pi}} e^{-\frac{\hat{\nu}^2}{2}}+\hat{\nu} -\hat{\nu}\erf\lp - \frac{\hat{\nu}}{\sqrt{2}}\rp - 2\hat{\nu}\delta
 =
  \frac{2}{\sqrt{2\pi}} e^{-\frac{\hat{\nu}^2}{2}}+\hat{\nu} -\hat{\nu}\erf\lp \erfinv \lp 1-2\delta\rp\rp - 2\hat{\nu}\delta \nonumber \\
 & = &
  \frac{2}{\sqrt{2\pi}} e^{-\frac{\hat{\nu}^2}{2}}+\hat{\nu} -\hat{\nu} \lp 1-2\delta \rp - 2\hat{\nu}\delta
  =  \frac{2}{\sqrt{2\pi}} e^{-\frac{\hat{\nu}^2}{2}}.
  \end{eqnarray}
From (\ref{eq:limlogpartfunsqrta0}), (\ref{eq:corthmsflrdt2eq3a0}), (\ref{eq:corthmsflrdt2eq1a0}),   
   and (\ref{eq:negthmprac1eq2}), for $\cZ=\mR^{mM}$ and $\cY= y\mS^{mM}$, one  first has 
\begin{eqnarray}
-f_{sq}(\infty)   
& = &
 \lim_{n\rightarrow\infty}\frac{\mE_G \min_{\x\in\cX,\z\in\cZ} \max_{\y\in\cY} \lp\y^TG\x +f(\y;\z)\rp }{\sqrt{n}},
\nonumber \\
& = &
 \lim_{n\rightarrow\infty}\frac{\mE_G \min_{\x\in\cX(\delta),\z\in\mR^{mM}} \max_{y>0,\y\in y\mS^{mM}} \lp\y^TG\x +f(\y;\z)\rp }{\sqrt{n}}
\nonumber \\
& = &
 \lim_{n\rightarrow\infty}\frac{\mE_G  \psi_{rp} }{\sqrt{n}} ,
  \label{eq:keyeq1}
\end{eqnarray}
and then as a first lifting level estimate
\begin{eqnarray}
-f_{sq}^{(1)}(\infty)   
& = &   \lim_{n\rightarrow\infty} \bar{\psi}_{rd}^{(1)}(\hat{\p},\hat{\q},\hat{\c},\nu,\delta)  
  \nonumber \\
  & \rightarrow &
   \lim_{n\rightarrow\infty}
  -  \mE_{{\mathcal U}_2}  \min_{\z\in\mR^{mM},\nu}  \max_{y>0}  
\Bigg (\Bigg . -\frac{1}{\sqrt{n}}\sqrt{\z^T I_{m\times m}\otimes ((1-r^2)I_{M\times M} +r^2\1_M\1_M^T)\z}
\nonumber \\
& & +   y \lp 
 \frac{1}{\sqrt{n}}\left \| \sum_{k=2}^{2}b_k\u^{(2,k)}  - \z \right \|_2 -f_1(\nu) \rp \Bigg .\Bigg )
  \nonumber \\
  & \rightarrow &
 \lim_{n\rightarrow\infty}  -  \mE_{{\mathcal U}_2}  \min_{\z\in\mR^{mM}}  \max_{y>0}  
\Bigg (\Bigg . -\frac{1}{\sqrt{n}}\sqrt{\z^T I_{m\times m}\otimes ((1-r^2)I_{M\times M} +r^2\1_M\1_M^T)\z}
\nonumber \\
& & +   y \lp 
 \frac{1}{\sqrt{n}}\left \| \sqrt{1-0}\u^{(2,2)}  - \z \right \|_2 -f_1(\hat{\nu}) \rp \Bigg .\Bigg )
 \nonumber \\
  & \rightarrow &
   \lim_{n\rightarrow\infty}
  -  \mE_{{\mathcal U}_2}  \min_{\z\in\mR^{mM}}  \max_{y>0}  
\Bigg (\Bigg . -\frac{1}{\sqrt{n}}\sqrt{\z^T I_{m\times m}\otimes ((1-r^2)I_{M\times M} +r^2\1_M\1_M^T)\z}
\nonumber \\
& & +   y \lp 
 \frac{1}{\sqrt{n}}\left \| \u^{(2,2)}  - \z \right \|_2 
 - \frac{2}{\sqrt{2\pi}} e^{-\frac{\hat{\nu}^2}{2}} \rp \Bigg .\Bigg ).
  \label{eq:keyeq2}
\end{eqnarray}
The above can alternatively be viewed as the following optimization
\begin{eqnarray}
-f_{sq}^{(1)}(\infty)   
    \rightarrow 
   \lim_{n\rightarrow\infty}
  - \mE_{{\mathcal U}_2}  \min_{\z\in\mR^{mM}}  & &    
  -\frac{1}{\sqrt{n}}\sqrt{\z^T I_{m\times m}\otimes ((1-r^2)I_{M\times M} +r^2\1_M\1_M^T)\z}
\nonumber \\
\mbox{subject to } & &     
 \frac{1}{\sqrt{n}}\left \| \u^{(2,2)}  - \z \right \|_2 
 \leq  \frac{2}{\sqrt{2\pi}} e^{-\frac{\hat{\nu}^2}{2}} = f_1(\hat{\nu}).
  \label{eq:keyeq3}
\end{eqnarray}
 From (\ref{eq:keyeq3}) concentrations allow to write
\begin{eqnarray}
 \lp -f_{sq}^{(1)}(\infty) \rp^2  
    \rightarrow 
   \lim_{n\rightarrow\infty}
  - \mE_{{\mathcal U}_2}  \min_{\z\in\mR^{mM}} & &     
  -\frac{1}{n} \z^T I_{m\times m}\otimes ((1-r^2)I_{M\times M} +r^2\1_M\1_M^T)\z
\nonumber \\
\mbox{subject to } & &     
 \frac{1}{n}\left \| \u^{(2,2)}  - \z \right \|_2^2 
 \leq   (f_1(\hat{\nu}))^2.
  \label{eq:keyeq4}
\end{eqnarray}
We then write the Lagrangian
\begin{eqnarray}
    \cL(\z,\lambda) = 
  -\frac{1}{n} \z^T I_{m\times m}\otimes ((1-r^2)I_{M\times M} +r^2\1_M\1_M^T)\z
+ \lambda\lp  \frac{1}{n}\left \| \u^{(2,2)}  - \z \right \|_2^2 
 -   (f_1(\hat{\nu}))^2 \rp.
  \label{eq:keyeq6}
\end{eqnarray}
and obtain
\begin{eqnarray}
 \lp -f_{sq}^{(1)}(\infty) \rp^2  
   & \rightarrow &
     \lim_{n\rightarrow\infty}
 - \mE_{{\mathcal U}_2}  \min_{\z\in\mR^{mM}} \max_{\lambda\geq 0} \cL(\z,\lambda) 
  =
    \lim_{n\rightarrow\infty}
   - \mE_{{\mathcal U}_2}  \max_{\lambda\geq 0}  \min_{\z\in\mR^{mM}}\cL(\z,\lambda)  ,
  \label{eq:keyeq5}
\end{eqnarray}
with equality enabled by the strong duality.  Additional algebraic transformations give
\begin{eqnarray}
    \cL(\z,\lambda) & = &  
  -\frac{1}{n} \z^T I_{m\times m}\otimes ((1-r^2)I_{M\times M} +r^2\1_M\1_M^T)\z
+ \frac{1}{n}  \lambda  \| \u^{(2,2)}\|_2^2  + \frac{1}{n} \lambda \|\z\|_2^2 
\nonumber \\
& & -2\frac{1}{n} \lambda \lp \u^{(2,2)} \rp^T\z 
 - \lambda  (f_1(\hat{\nu}))^2 .
  \label{eq:keyeq7}
\end{eqnarray}
 After setting 
\begin{eqnarray}
\z^{(i)}= \z_{(i-1)M+1:iM} \quad \mbox{ and } \quad  \u^{(2,2;i)}= \u_{(i-1)M+1:iM}^{(2,2)} ,
  \label{eq:keyeq8}
\end{eqnarray}
 we rewrite (\ref{eq:keyeq7}) as
\begin{eqnarray}
    \cL(\z,\lambda) & = &  
  -\frac{1}{n} \sum_{i=1}^{m} \lp \z^{(i)}\rp^T ((1-r^2)I_{M\times M} +r^2\1_M\1_M^T)\z^{(i)}
+ \frac{1}{n}  \sum_{i=1}^{m} \lambda  \| \u^{(2,2;i)}\|_2^2  + \frac{1}{n}  \sum_{i=1}^{m}\lambda \|\z^{(i)}\|_2^2 
\nonumber \\
& &  -2\frac{1}{n} \lambda   \sum_{i=1}^{m} \lp \u^{(2,2;i)} \rp^T\z^{(i)} 
 - \lambda  (f_1(\hat{\nu}))^2 .
  \label{eq:keyeq9}
\end{eqnarray}
Eigen-decomposition further gives
\begin{eqnarray}
 (1-r^2)I_{M\times M} +r^2\1_M\1_M^T 
 = U\Sigma U^T, 
  \label{eq:keyeq10}
\end{eqnarray}
 with
 \begin{eqnarray}
 UU^T=I_{M\times M} \quad \mbox{ and } \quad  \sigma=\{1-r^2,1-r^2,\dots, 1-r^2,1+(M-1)r^2 \}. 
  \label{eq:keyeq11}
\end{eqnarray}
Change of variables $\z^{(i)}\rightarrow U^T\z^{(i)}$ then gives
\begin{eqnarray}
    \cL(\z,\lambda) & = &  
  -\frac{1}{n} \sum_{i=1}^{m} \lp \z^{(i)}\rp^T \mathrm{diag}(\sigma)\z^{(i)}
+ \frac{1}{n}  \sum_{i=1}^{m} \lambda  \| U^T\u^{(2,2;i)}\|_2^2  + \frac{1}{n}  \sum_{i=1}^{m}\lambda \|\z^{(i)}\|_2^2 
\nonumber \\
& &  -2\frac{1}{n} \lambda   \sum_{i=1}^{m} \lp U^T\u^{(2,2;i)} \rp^T \z^{(i)} 
 - \lambda  (f_1(\hat{\nu}))^2 
\nonumber \\
& = &   \frac{1}{n} \sum_{i=1}^{m} \lp \z^{(i)}\rp^T  (\lambda I - \mathrm{diag}(\sigma))\z^{(i)}
+ \frac{1}{n}  \sum_{i=1}^{m} \lambda  \| U^T\u^{(2,2;i)}\|_2^2    
\nonumber \\
& &  -2\frac{1}{n} \lambda   \sum_{i=1}^{m} \lp U^T\u^{(2,2;i)} \rp^T \z^{(i)} 
 - \lambda  (f_1(\hat{\nu}))^2 .
  \label{eq:keyeq12}
\end{eqnarray}
After taking the $\z^{(i)}$ derivatives, one obtains
\begin{eqnarray}
   \frac{ d \cL(\z,\lambda)}{d\z^{(i)}} & = &  
  \frac{2}{n}   (\lambda I -  \mathrm{diag}(\sigma)) \z^{(i)}
 -2\frac{1}{n} \lambda     U^T\u^{(2,2;i)}  .
  \label{eq:keyeq13}
\end{eqnarray}
Equalling the above derivative to zero gives for the optimal $\z^{(i)}$
\begin{eqnarray}
    \hat{\z}^{(i)} =
   \lambda    (\lambda I -  \mathrm{diag}(\sigma))^{-1}  U^T\u^{(2,2;i)}  .
  \label{eq:keyeq14}
\end{eqnarray}
Plugging this back into (\ref{eq:keyeq12}) gives
\begin{equation}
    \min_{\z\in\mR^{mM}}\cL(\z,\lambda)  
 =    -\frac{\lambda^2}{n} \sum_{i=1}^{m} \lp  U^T\u^{(2,2;i)} \rp^T  (\lambda I - \mathrm{diag}(\sigma))^{-1}  U^T\u^{(2,2;i)} 
+ \frac{1}{n}  \sum_{i=1}^{m} \lambda  \| U^T\u^{(2,2;i)}\|_2^2    
  - \lambda  (f_1(\hat{\nu}))^2 ,
  \label{eq:keyeq15}
\end{equation}
and 
\begin{equation}
 \mE_{{\mathcal U}_2}      \min_{\z\in\mR^{mM}}\cL(\z,\lambda)  
 =    -\frac{ m}{n} \sum_{i=1}^{M} \frac{\lambda^2}{\lambda - \sigma_i} 
+ \frac{m}{n} \lambda M   
  - \lambda  (f_1(\hat{\nu}))^2 ,
  \label{eq:keyeq16}
\end{equation}
Combining (\ref{eq:keyeq5}) and  (\ref{eq:keyeq16})   and relying on concentrations, we find 
\begin{eqnarray}
 \lp -f_{sq}^{(1)}(\infty) \rp^2  
   & \rightarrow &
   \lim_{n\rightarrow\infty}
     - \mE_{{\mathcal U}_2}  \max_{\lambda\geq \max(\sigma_i)}  \min_{\z\in\mR^{mM}}\cL(\z,\lambda)  
\nonumber \\
&     \rightarrow  &
      \max_{\lambda\geq \max(\sigma_i)}
    \lp -\alpha\sum_{i=1}^{M} \frac{\lambda^2}{\lambda - \sigma_i} 
+ \alpha\lambda M   
  - \lambda  (f_1(\hat{\nu}))^2 \rp
\nonumber \\
&     \rightarrow  &
      \max_{\lambda\geq \max(\sigma_i)}
\lp
     -\frac{\alpha (M-1)\lambda^2}{\lambda - (1-r^2)} 
- \frac{\alpha\lambda^2}{\lambda - (1+(M-1)r^2)}      
+ \alpha\lambda M   
  - \lambda  (f_1(\hat{\nu}))^2 \rp.
  \label{eq:keyeq17}
\end{eqnarray}

\subsection{Numerical results}
\label{sec:numres}

The above is sufficient to determine critical Hopfield-Hebbian dynamics capacity estimate on the first lifting level, $\alpha^{(1)}_c(M,r)$. In particular, keeping in mind (\ref{eq:ham0a1}), (\ref{eq:limlogpartfunsqrta0}), and (\ref{eq:corthmsflrdt2eq1a0b0}), one first has that  $\lp -f_{sq}^{(1)}(\infty) \rp^2 $ from (\ref{eq:keyeq17}) is the first lifting level estimate for $\xi(\delta)^2$. Utilizing (\ref{eq:ex0a14b5}) and (\ref{eq:ex0a16}), one then determines the $\alpha^{(1)}_c(M,r)$ as a maximal $\alpha$ such that $\xi_1(\delta)=-(1-2\delta)^2r^2 M -\xi(\delta)^2\approx
-(1-2\delta)^2r^2M- \lp -f_{sq}^{(1)}(\infty) \rp^2 $ is monotonically decreasing in $\delta$. The obtained results are shown in Figure \ref{fig:fig1}. To be in agreement with the existing literature  \cite{AglMFV26}, we take $n=1000$ and consider the whole range for correlation parameter $r=[0,1]$ in  the so-called full-rank regime with $\alpha M=2$ (where the \cite{AglMFV26}'s spectral methods slightly deviate from the practical dynamics performance). The colored values represent the average overlap of the dynamics output with the archetype ground truth assuming that one start in its vicinity (with the overlap $0.95$), whereas the full curve indicates the above theoretical prediction (for completeness, the spectral methods estimates obtained in \cite{AglMFV26} are shown in the figure as well). As can be seen, the theoretical predictions fairly closely follow the simulated results even for fairly small problem dimensions. Moreover, in the low $r$ regime, our results closely approach spectral predictions of \cite{AglMFV26}, confirming that they are indeed fairly accurate approximations.

We also mention that implementation of higher lifting levels is in principle doable but requires significant numerical effort. However, keeping in mind that in uncorrelated scenarios lifting effect after the first level is for all practical purposes negligible, pursuing higher lifting levels may be unnecessary.  At the very least in the highly correlated regimes with $r$ close to $1$ (but quite likely in the entire $r\in[0,1]$ range as well), we do not expect much of further corrections on higher levels.

 \begin{figure}[h]
\centering
\centerline{\includegraphics[width=1\linewidth]{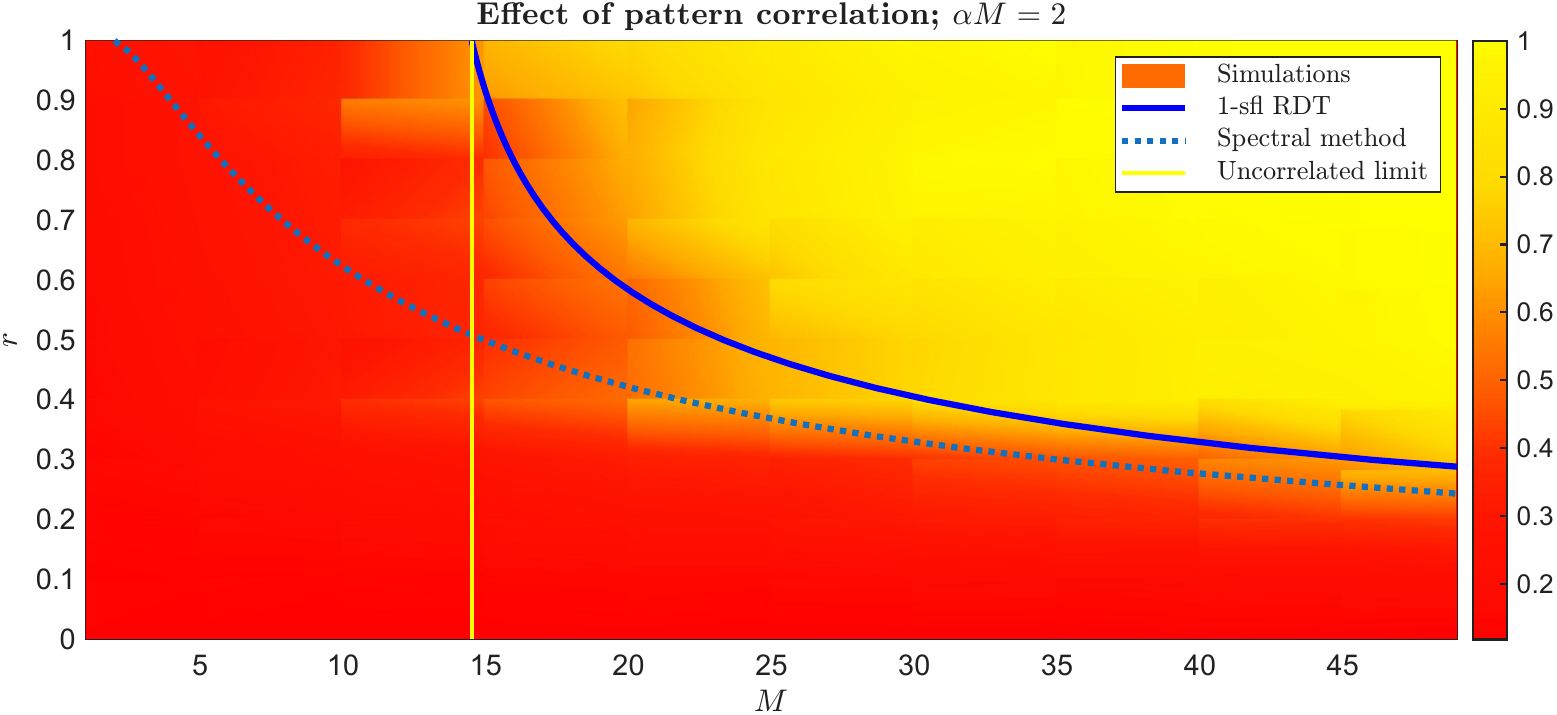}}
\caption{Capacity phase transition in Hopfield-Hebbian dynamics with correlated patterns}
\label{fig:fig1}
\end{figure}

\subsection{Special case $r=1$ -- uncorrelated patterns}
\label{sec:specuncor}

In Figure \ref{fig:fig1} we also show (as a solid yellow line) the limit one would obtain on the first lifting level in case of uncorrelated patterns. It is roughly given as $M\approx \frac{2}{0.1379}\approx 14.5033$, where $\approx 0.1379$ is the famous estimate given via replica methods (assuming replica symmetry ansatz) in \cite{AmiGutSom85} and reobtained on the first level of the fully lifted RDT in \cite{Stojnichebhop24}. The results given in figure indicate that as $r\rightarrow 1$ the above predictions actually approach $M\approx 14.5033$. We analytically confirm below that this is indeed the case.

Taking $M=1$ in (\ref{eq:keyeq17}), we have
\begin{eqnarray}
 \lp -f_{sq}^{(1)}(\infty) \rp^2  
 &     \rightarrow  &
      \max_{\lambda\geq 1}
\lp
- \frac{\alpha\lambda^2}{\lambda - 1}      
+ \alpha\lambda    
  - \lambda  (f_1(\hat{\nu}))^2 \rp \triangleq \max_{\lambda > 1}\zeta(\lambda).
  \label{eq:keyeq18}
\end{eqnarray}
Taking the derivative with respect to $\lambda$ gives
\begin{eqnarray}
 \frac{d\zeta(\lambda)}{d\lambda}=
 - \frac{2\alpha\lambda}{\lambda - 1}      
 + \frac{\alpha\lambda^2}{(\lambda - 1)^2}      
+ \alpha    
  -  (f_1(\hat{\nu}))^2   .
  \label{eq:keyeq19}
\end{eqnarray}
Equalling the above derivative to zero gives 
\begin{eqnarray}
\frac{\lambda}{\lambda -1} = 1+
    \frac{f_1(\hat{\nu})}{\sqrt{\alpha}}  ,
  \label{eq:keyeq20}
\end{eqnarray}
and optimal $\lambda$
\begin{eqnarray}
\hat{\lambda} = 1+
    \frac{\sqrt{\alpha}}{f_1(\hat{\nu})}  .
  \label{eq:keyeq21}
\end{eqnarray}
From (\ref{eq:keyeq18}), one then also has
\begin{equation}
 \lp -f_{sq}^{(1)}(\infty) \rp^2  
      \rightarrow  
 \zeta(\hat{\lambda}) = -\frac{\alpha\hat{\lambda}}{\hat{\lambda} -1} -\hat{\lambda} (f_1(\hat{\nu}))^2 =
 -\alpha -  \sqrt{\alpha}f_1(\hat{\nu}) - (f_1(\hat{\nu}))^2 - \sqrt{\alpha} f_1(\hat{\nu})
 = -( \sqrt{\alpha} + f_1(\hat{\nu}))^2 .
  \label{eq:keyeq22}
\end{equation}
Keeping in mind (\ref{eq:negprac19a4}) and recalling $\hat{\nu}=-\sqrt{2} \erfinv \lp 1-2\delta\rp$, one also finds
\begin{eqnarray} 
 \lp -f_{sq}^{(1)}(\infty) \rp^2  
      \rightarrow  
  -\lp \sqrt{\alpha} +  \frac{2}{\sqrt{2\pi}} e^{-\frac{\hat{\nu}^2}{2}} \rp^2 .
  \label{eq:keyeq23}
  \end{eqnarray}
Utilizing (\ref{eq:ex0a14b5}) and (\ref{eq:ex0a16}), one then determines the $\alpha^{(1)}_c(1,1)$ as a maximal $\alpha$ such that $\xi_1(\delta)=-(1-2\delta)^2 -\xi(\delta)^2\approx
-(1-2\delta)^2- \lp -f_{sq}^{(1)}(\infty) \rp^2 $ is monotonically decreasing in $\delta$. This is precisely what is done in \cite{Stojnichebhop24} where the above mentioned $\approx 0.1379$ estimate is obtained.

Now, taking $r=1$ with a general $M$ in (\ref{eq:keyeq17}) gives
\begin{eqnarray}
 \lp -f_{sq}^{(1)}(\infty) \rp^2  
 &     \rightarrow  &
      \max_{\lambda\geq M}
\lp
     -\frac{\alpha (M-1)\lambda^2}{\lambda } 
- \frac{\alpha\lambda^2}{\lambda - M}      
+ \alpha\lambda M   
  - \lambda  (f_1(\hat{\nu}))^2 \rp
\nonumber \\
   &   =  &
      \max_{\lambda\geq M}
\lp
 - \frac{\alpha\lambda^2}{\lambda - M}      
+ \alpha\lambda    
  - \lambda  (f_1(\hat{\nu}))^2 \rp
\nonumber \\
   &   =  &
      \max_{\lambda\geq M}
\lp
 - \frac{\alpha M\lambda^2}{M^2\lp\frac{\lambda}{M} - 1\rp}      
+ \alpha M\frac{\lambda}{M}    
  - M \frac{\lambda}{M}  (f_1(\hat{\nu}))^2 \rp
\nonumber \\
   &   =  &
M      \max_{\lambda\geq M}
\lp
 - \frac{\alpha \lambda^2}{M^2\lp\frac{\lambda}{M} - 1\rp}      
+ \alpha \frac{\lambda}{M}    
  -  \frac{\lambda}{M}  (f_1(\hat{\nu}))^2 \rp
  .
  \label{eq:keyeq24}
\end{eqnarray}
Change of variables $\lambda\rightarrow \frac{\lambda}{M}$ allows to rewrite (\ref{eq:keyeq24}) as
\begin{eqnarray}
 \lp -f_{sq}^{(1)}(\infty) \rp^2  
 &     \rightarrow  &
        \max_{\lambda\geq 1}
\lp
 - \frac{\alpha \lambda^2}{\lambda - 1}      
+ \alpha \lambda    
  -  \lambda  (f_1(\hat{\nu}))^2 \rp
\triangleq \max_{\lambda > 1}\zeta(\lambda)  ,
  \label{eq:keyeq25}
\end{eqnarray}
which is structurally identical to (\ref{eq:keyeq18}). Consequently, all of the above analysis applies with the same concluding estimate $\approx 0.1379$. This then analytically confirms that full yellow and blue curves in Figure \ref{fig:fig1} indeed coincide for $r=1$. In other words, the machinery that we have developed for uncorrelated patterns -- when specialized to quality of correlation parameters that actually correspond to the no correlation scenario -- matches the corresponding known results for the uncorrelated patterns.

\section{Conclusion}
\label{sec:conc}

Over the last several decades, the classical Hopfield-Hebbian dynamics \cite{Hop82} has been thoroughly studied in idealized scenarios with uncorrelated patterns.   Such settings allow  for proper understanding of main concepts and are sufficiently simple to enable analytical tractability. However,  practical data are rarely uncorrelated and consequently significantly more challenging to analyze precisely. Even implementation of traditional statistical mechanics methods is a demanding task in those situations. 

We consider the so-called archetype/examples type of correlations where each of $m$ \emph{archetype} ground truth patterns is replaced with $M$ correlated \emph{examples} with parameter $r$ governing the quality of correlation. A powerful Random Duality Theory (RDT) based framework is developed to analyze correlated data. The  capacity characterizations already on the first level of RDT lifting are in an excellent agreement with the simulated dynamics and precisely determine correlation parameters $M$ and $r$ effects on the dynamics phase-transitioning behavior. Moreover, they very well complement the existing literature. In particular, we confirm that \cite{AglMFV26}'s elegant spectral characterizations are excellent approximations in the low-correlated (small $r$) regime. We also close the gap in the highly correlated regime where spectral methods tend to deviate from the simulated dynamics.

\begin{singlespace}
\bibliographystyle{plain}
\bibliography{nflgscompyxRefs}
\end{singlespace}

\appendix

\section{Proof of Corollary \ref{cor:corsflrdt2}}
\label{sec:app0}

Utilizing results of Corollary \ref{cor:corsflrdt1}, we have the following for the \emph{random dual}
\begin{align}\label{eq:prac1}
    \psi_{rd}(\p,\q,\c)   & \triangleq    \frac{y^2}{2}    \sum_{k=2}^{r+1}\Bigg(\Bigg.
   \p_{k-1}\q_{k-1}
   -\p_{k}\q_{k}
  \Bigg.\Bigg)
\c_k
  - \bar{\psi}_{S,\infty}(f(\y),\cX(\delta),y\mS^{mM},\p,\q,\c,s) 
  \nonumber \\
  & =   \frac{y^2}{2}    \sum_{k=2}^{r+1}\Bigg(\Bigg.
   \p_{k-1}\q_{k-1}
   -\p_{k}\q_{k}
  \Bigg.\Bigg)
\c_k
  - \frac{1}{n}\varphi(D^{(per)}(s),
  \c) - \frac{1}{n}\varphi(D^{(sph)}(s),\c)
    \nonumber \\
  & =   \frac{y^2}{2}    \sum_{k=2}^{r+1}\Bigg(\Bigg.
   \p_{k-1}\q_{k-1}
   -\p_{k}\q_{k}
  \Bigg.\Bigg)
\c_k
  - \frac{1}{n}\varphi(D^{(tot)}(s),
  \c) , 
  \end{align}
where, based on (\ref{eq:fl4})-(\ref{eq:fl5}),
  \begin{eqnarray}\label{eq:prac2a0}
\varphi(D,\c) & \triangleq &
 \mE_{G,{\mathcal U}_{r+1}} \frac{1}{\c_r} \log
\lp \mE_{{\mathcal U}_{r}} \lp \dots \lp \mE_{{\mathcal U}_3}\lp\lp\mE_{{\mathcal U}_2} \lp
\lp    e^{D}   \rp^{\c_2}\rp\rp^{\frac{\c_3}{\c_2}}\rp\rp^{\frac{\c_4}{\c_3}} \dots \rp^{\frac{\c_{r}}{\c_{r-1}}}\rp, \nonumber \\
  \end{eqnarray}
and
\begin{eqnarray}\label{eq:prac3}
D^{(per)}(s) & = & \max_{\x\in\cX(\delta)} \lp   s\sqrt{n}   y   \lp\sum_{k=2}^{r+1}c_k\h^{(k)}\rp^T\x  \rp 
\nonumber \\
  D^{(sph)}(s) & = &  \max_{\z\in\mR^n}  s \max_{\y\in\mS^m}
\lp  \sqrt{n} f(\y) + \sqrt{n}  \y^T\lp\sum_{k=2}^{r+1}b_k\u^{(2,k)}\rp \rp
\nonumber \\
  D^{(tot)}(s) & =  &  D^{(per)}(s) + D^{(sph)}(s) .
 \end{eqnarray}

\noindent  \underline{\textbf{\emph{(i) Handling $D^{(per)}(-1)$:}}} This is presented in \cite{Stojnichebhop24}. For completeness, we repeat some of the key steps. We start with cosmetic change $\h^{(k)} \rightarrow -\h^{(k)}$  (which, due to the symmetry of $\h^{(k)}$, is of no statistical relevance) 
\begin{eqnarray}\label{eq:prac4a0}
D^{(per)}(-1) & = &  y\max_{\x\in\cX(\delta)}   \lp \sqrt{n}      \lp\sum_{k=2}^{r+1}c_k\h^{(k)}\rp^T\x \rp =
 y\sqrt{n} \max_{\x\in\cX(\delta)}   \lp      \lp\sum_{k=2}^{r+1}c_k\h^{(k)}\rp^T\x \rp.
 \end{eqnarray}
The following $\cX$ parametrization allows to reformulate (\ref{eq:ex0a15}) as
\begin{eqnarray}
\cX(\delta) & \triangleq & \left \{\x \,\, | \quad \|\x-\bg^{(1)}\|_2^2=4\delta,\x\in S_H^n\right \} \nonumber \\
 & = & \left \{\x \,\, | \quad \left \|\x-\frac{1}{\sqrt{n}}\right \|_2^2=4\delta,\x\in S_H^n\right \} 
\nonumber \\
& = & \left \{\x \,\, | \quad \x=\frac{\1}{\sqrt{n}}-\z,\z\in \left \{0,\frac{2}{\sqrt{n}} \right \}^n, \sum_{i=1}^{n}\z_i=2\delta\sqrt{n}\right \},
 \label{eq:prac4a0a1}
\end{eqnarray}
where we recall $\bg^{(1)}=\frac{\1}{\sqrt{n}}$. Combining  (\ref{eq:prac4}) and (\ref{eq:prac4a0a1}) we also find
\begin{eqnarray}\label{eq:prac4a0a2}
D^{(per)}(-1) & = &  -\sqrt{n} y \min_{\z}   \lp     - \lp\sum_{k=2}^{r+1}c_k\h^{(k)}\rp^T\lp \frac{\1}{\sqrt{n}} -\z\rp \rp \nonumber \\
\mbox{subject to} & &  \z\in \left \{0,\frac{2}{\sqrt{n}} \right \}^n, \sum_{i=1}^{n}\z_i=2\delta\sqrt{n}.
 \end{eqnarray}
Writing the Lagrangian gives
\begin{eqnarray}\label{eq:prac4a0a3}
D^{(per)}(-1) & =  &  -\sqrt{n} y \min_{\z\in \left \{0,\frac{2}{\sqrt{n}} \right \}^n} \max_{\nu}
  \lp     - \lp\sum_{k=2}^{r+1}c_k\h^{(k)}\rp^T \lp\frac{\1}{\sqrt{n}}-\z\rp -\nu \sum_{i=1}^{n}\z_i+2\nu\delta\sqrt{n} \rp \nonumber \\
  & =  &  -\sqrt{n} y \max_{\nu} \min_{\z\in \left \{0,\frac{2}{\sqrt{n}} \right \}^n}
  \lp     - \lp\sum_{k=2}^{r+1}c_k\h^{(k)}\rp^T \lp\frac{\1}{\sqrt{n}}-\z\rp -\nu \sum_{i=1}^{n}\z_i+2\nu\delta\sqrt{n} \rp,
 \end{eqnarray}
where the second equality is implied by the strong Lagrangian duality. Optimizing over $\z$, we further find
\begin{eqnarray}\label{eq:prac4a0a4}
D^{(per)}(-1) & =  & - y\max_{\nu} \lp 2\nu\delta n  + 2 \min\lp \sum_{k=2}^{r+1}c_k\h^{(k)} -\nu,0 \rp  - \lp\sum_{k=2}^{r+1}c_k\h^{(k)}\rp^T\1 \rp
\nonumber \\
  & =  &  y\min_{\nu} \lp  -2\nu\delta n  - 2 \min\lp \sum_{k=2}^{r+1}c_k\h^{(k)} -\nu,0 \rp  + \lp\sum_{k=2}^{r+1}c_k\h^{(k)}\rp^T\1 \rp \nonumber \\
   & =  & y \min_{\nu} \lp  -2\nu\delta n  + \sum_{i=1}^nD^{(per)}_i(c_k) \rp,
 \end{eqnarray}
where
\begin{eqnarray}\label{eq:prac5}
D^{(per)}_i(c_k)=-2\min \lp\sum_{k=2}^{r+1}c_k\h_i^{(k)}-\nu,0\rp +\sum_{k=2}^{r+1}c_k\h_i^{(k)}.
\end{eqnarray}

\noindent \underline{\textbf{\emph{(ii) Handling $D^{(sph)}(-1)$:}}}  We first observe
\begin{align}\label{eq:prac7}
   D^{(sph)}(-1) & \triangleq     \sqrt{n} \max_{\z\in\mR^n} - \max_{\y\in y\mS^{mM}}
\lp f(\y)+  \y^T\lp\sum_{k=2}^{r+1}b_k\u^{(2,k)}\rp \rp
\nonumber \\
& =  
-\sqrt{n} \min_{\z\in\mR^n}  \max_{\y\in y\mS^{mM}}
\lp -\sqrt{\z^T I_{m\times m}\otimes ((1-r^2)I_{M\times M} +r^2\1_M\1_M^T)\z}
-\y^T\z
 +  \y^T\lp\sum_{k=2}^{r+1}b_k\u^{(2,k)}\rp \rp
 \nonumber \\
& = 
  - \sqrt{n} \min_{\z\in\mR^n} 
\lp -\sqrt{\z^T I_{m\times m}\otimes ((1-r^2)I_{M\times M} +r^2\1_M\1_M^T)\z}
+  y  \left \| \sum_{k=2}^{r+1}b_k\u^{(2,k)}  - \z \right \|_2  \rp.
 \end{align}
A combination of (\ref{eq:prac4a0a4}) and (\ref{eq:prac7}) gives
\begin{eqnarray}\label{eq:prac7a0}
   D^{(tot)}(-1) & =  &  
     -  \min_{\z\in\mR^n} 
\Bigg (\Bigg . -\sqrt{n}\sqrt{\z^T I_{m\times m}\otimes ((1-r^2)I_{M\times M} +r^2\1_M\1_M^T)\z}
+ \sqrt{n} y \Bigg (\Bigg .
 \left \| \sum_{k=2}^{r+1}b_k\u^{(2,k)}  - \z \right \|_2  
\nonumber \\
& & 
-\min_{\nu} \lp  -2\nu\delta n  + \sum_{i=1}^nD^{(per)}_i(c_k) \rp \Bigg . \Bigg )
\Bigg . \Bigg ).
 \end{eqnarray}
 Comparing (\ref{eq:prac7a0}) to (\ref{eq:prac4}) completes the proof of Corollary \ref{cor:corsflrdt2}
and is sufficient for what we eventually need. For the completeness, we also show below the remaining decoupling over $\z$. Utilizing the \emph{square root trick}, we find
\begin{align}\label{eq:prac8}
   D^{(sph)} (-1)
&=  -\sqrt{n} \min_{\z\in\mR^n,\gamma_{z},\gamma} 
\Bigg (\Bigg . \frac{-(1-r^2)\sum_{i=1}^{mM}\z_i^2 +r^2 \sum_{i=1}^{m}\min_{\gamma_{z,i}} \lp  -2\gamma_{z,i} \1_M^T\z_{(i-1)M+1:iM} +\gamma_{z,i}^2 \rp }{4\gamma_{z}} -\gamma_{z}
\nonumber \\
& \hspace{.15in} +
 \frac{\left \|\sum_{k=2}^{r+1}b_k\u^{(2,k)} -\z\right \|_2^2}{4\gamma}+\gamma  
\Bigg .\Bigg ) 
\nonumber \\
&=  -\sqrt{n} \min_{\z\in\mR^n,\gamma_{z},\gamma,\gamma_{z,i}} 
\Bigg (\Bigg . \frac{-(1-r^2)\sum_{i=1}^{mM}\z_i^2 +r^2 \sum_{i=1}^{m}  \lp  -2\gamma_{z,i} \1_M^T\z_{(i-1)M+1:iM} +\gamma_{z,i}^2 \rp}{4\gamma_z} -\gamma_z +
\nonumber \\
& \hspace{.15in} +
  \frac{\sum_{i=1}^{mM}\lp\sum_{k=2}^{r+1}b_k\u_i^{(2,k)} -\z_i \rp^2}{4\gamma}+\gamma  
\Bigg .\Bigg ) 
,
 \end{align}
 which provides full decoupling and is trivial to incorporate in $D^{(tot)}(-1)$.

\end{document}